\documentclass[11pt,reqno]{amsart}
\usepackage[T1]{fontenc}

\usepackage{graphicx,enumerate,amssymb,bbold}
\usepackage[dvipsnames]{xcolor}
\usepackage{tikz}
\usepackage{tikz-cd}
\usepackage{quiver}
\usepackage{booktabs}
\usepackage{float}
\usepackage{framed}
\usepackage[normalem]{ulem}
\usepackage[textwidth=2cm,textsize=tiny]{todonotes}
\usepackage[notrig]{physics}
\usepackage[font=footnotesize,labelfont=small]{subcaption}
\usepackage[ruled,vlined]{algorithm2e}
\usepackage[left=2.5cm,right=2.5cm,top=2.5cm,bottom=2.5cm,includeheadfoot]{geometry}

\usepackage[colorlinks=true, pdfstartview=FitV, linkcolor=blue,
            citecolor=blue, urlcolor=blue]{hyperref}

\numberwithin{equation}{section}
\numberwithin{figure}{section}
\usepackage{xcolor}
\newcommand\crule[3][black]{\textcolor{#1}{\rule{#2}{#3}}}
\definecolor{c1}{RGB}{228,26,28}
\definecolor{c2}{RGB}{55,126,184}
\definecolor{c3}{RGB}{77,175,74}
\definecolor{c4}{RGB}{152,78,163}
\definecolor{c5}{RGB}{255,127,0}
\definecolor{c6}{RGB}{255,255,51}
\definecolor{c7}{RGB}{166,86,40}
\definecolor{c8}{RGB}{247,129,191}

\theoremstyle{plain}
\newtheorem{theorem}{Theorem}[section]
\newtheorem{lemma}[theorem]{Lemma}
\newtheorem{proposition}[theorem]{Proposition}

\theoremstyle{definition}

\newtheorem{remark}[theorem]{Remark}

\newcommand{\bitem}{\begin{itemize}}
\newcommand{\eitem}{\end{itemize}}
\newcommand{\mc}[1]{\mathcal{#1}}

\newcommand{\N}{\mathbb{N}}
\newcommand{\R}{\mathbb{R}}

\newcommand{\bpm}{\begin{pmatrix}}
\newcommand{\epm}{\end{pmatrix}}
\newcommand{\bvm}{\begin{vmatrix}}
\newcommand{\evm}{\end{vmatrix}}
\newcommand{\bsm}{\left(\begin{smallmatrix}}
\newcommand{\esm}{\end{smallmatrix}\right)}
\newcommand{\T}{\top}

\newcommand{\wh}[1]{\widehat{#1}}
\newcommand{\wt}[1]{\widetilde{#1}}
\newcommand{\la}{\langle}
\newcommand{\ra}{\rangle}

\newcommand{\veps}{\varepsilon}

\newcommand{\w}{\omega}

\newcommand{\vphi}{\varphi}

\newcommand{\eins}{\mathbb{1}}

\DeclareMathSymbol{\mydiv}{\mathbin}{symbols}{"04}

\DeclareMathOperator{\Diag}{Diag}

\DeclareMathOperator{\supp}{supp}

\DeclareMathOperator{\vvec}{vec}

\DeclareMathOperator{\Exp}{Exp}
\DeclareMathOperator{\expm}{expm}

\newcommand{\sst}[1]{{\scriptscriptstyle #1}}

\makeatletter
\def\widebreve{\mathpalette\wide@breve}
\def\wide@breve#1#2{\sbox\z@{$#1#2$}%
     \mathop{\vbox{\m@th\ialign{##\crcr
\kern0.08em\brevefill#1{0.8\wd\z@}\crcr\noalign{\nointerlineskip}%
                    $\hss#1#2\hss$\crcr}}}\limits}
\def\brevefill#1#2{$\m@th\sbox\tw@{$#1($}%
  \hss\resizebox{#2}{\wd\tw@}{\rotatebox[origin=c]{90}{\upshape(}}\hss$}
\makeatletter

\newcommand*{\Reals}{\mathbb{R}}

\usepackage{aligned-overset}
\usepackage{cancel}

\usepackage{scalerel}

\title[Learning Linearized Assignment Flows]{Learning Linearized Assignment Flows for Image Labeling}
\author[A.~Zeilmann, S.~Petra, C.~Schn\"{o}rr]{Alexander Zeilmann, Stefania Petra, Christoph Schn\"{o}rr}
\address[A.~Zeilmann, C.~Schn\"{o}rr]{Image and Pattern Analysis Group, Heidelberg University, Germany}
\urladdr{\url{https://ipa.math.uni-heidelberg.de}}
\address[S.~Petra]{Mathematical Imaging Group, Heidelberg University, Germany}
\urladdr{\url{https://www.stpetra.com}}
\date{}

\begin{document}

\begin{abstract}
We introduce a novel algorithm for estimating optimal parameters of linearized assignment flows for image labeling. An exact formula is derived for the parameter gradient of any loss function that is constrained by the linear system of ODEs determining the linearized assignment flow. We show how to efficiently evaluate this formula using a Krylov subspace and a low-rank approximation. This enables us to perform parameter learning by Riemannian gradient descent in the parameter space, without the need to backpropagate errors or to solve an adjoint equation.
Experiments demonstrate that our method performs as good as highly-tuned machine learning software using automatic differentiation.
Unlike methods employing automatic differentiation, our approach yields a low-dimensional representation of internal parameters and their dynamics which helps to understand how assignment flows and more generally neural networks work and perform.
\end{abstract}


\maketitle
\tableofcontents

\section{Introduction}\label{sec:Introduction}

\subsection{Overview, Motivation}
Learning the parameters of large neural networks from training data constitutes a basic
problem in imaging science, machine learning and other fields.
The prevailing approach utilizes gradient descent or approximations thereof based on automatic differentiation \cite{Baydin:2018aa} and corresponding software tools, like PyTorch~\cite{Paszke2019} and TensorFlow~\cite{Abadi2015}. This kind of software support has been spurring research in imaging science and machine learning dramatically. However, merely relying on numerical schemes and their automatic differentiation tends to thwart attempts to shed light on the often-criticized black-box behavior of deep networks and to better understand the internal representation and function of parameters and their adaptive dynamics.

In this paper, we explore a different route. Adopting the \textit{linearized assignment flow} approach introduced by \cite{Zeilmann2020}, we focus on a corresponding large system of linear ODEs of the form
\begin{equation}\label{eq:LAF-intro}
\dot V = A(\Omega) V + B,
\end{equation}
and study a geometric approach to learning the regularization parameters $\Omega$ by Riemannian gradient descent of a loss function
\begin{equation}\label{eq:loss-intro}
\Omega\mapsto\mc{L}(V(T;\Omega))
\end{equation}
constrained by the dynamical system \eqref{eq:LAF-intro}. Here, we exploit the crucial property that the solution to \eqref{eq:LAF-intro} can be specified in closed form \eqref{eq:V-t-vphi} and can be computed efficiently using exponential integration (\cite{Zeilmann2020} and Section \ref{ssec:Exponential_Integration}). Matrix $V \in \R^{|I|\times c}$ represents a tangent vector of the so-called assignment manifold, $|I|$ is the number of nodes $i\in I$ of the underlying graph, and $c$ is the number of labels (classes) that have to be assigned to data observed at nodes $i\in I$. Specifically,
\begin{itemize}
\item
we derive a formula -- see Theorem \ref{thm:loss-function-gradient} -- for the Euclidean parameter gradient $\partial_{\Omega}\mc{L}(V(T;\Omega))$ in closed form;
\item
we show that a low-rank representation of this gradient can be used to efficiently and accurately approximate this closed form gradient; neither backpropagation, nor automatic differentiation or solving adjoint equations are required;
\item
we highlight that the resulting parameter estimation algorithm, in terms of a Riemannian gradient descent iteration \eqref{eq:R-descent-iteration} on the parameter manifold,  can be implemented without any specialized software support with modest
computational resources;
\end{itemize}
The significance of our work reported in this paper arises in a broader context. The linearized assignment flow approach also comprises the equation
\begin{equation}\label{eq:W-VT}
W(T) = \Exp_{\eins_{\mc{W}}}(V(T))
\end{equation}
that yields the labeling in terms of almost integral assignment vectors $W_{i}\in\R_{+}^{c},\; i\in I$ that form the rows of the matrix $W$, depending on the solution $V(t)$ of \eqref{eq:LAF-intro} for a sufficiently large time $t=T$. Both equations \eqref{eq:W-VT} and \eqref{eq:LAF-intro} together constitute a linearization of the full nonlinear assignment flow \cite{Astrom:2017ac}
\begin{equation}\label{eq:AF-intro}
\dot W = R_{W}S(W)
\end{equation}
at the barycenter $\eins_{\mc{W}}$ of the assignment manifold. Choosing an arbitrary sequence of time intervals (step sizes) $h_{1}, h_{2}, \dotsc$ and setting
\begin{equation}
W^{(0)}=\eins_{\mc{W}},\qquad
W^{(k)}=W(h_{k}),\qquad k\in\N,
\end{equation}
a sequence of linearized assignment flows
\begin{subequations}\label{eq:exp-int-AF-intro}
\begin{align}
W^{(k+1)} &= \Exp_{\eins_{\mc{W}}}(V^{(k)}),
\label{eq:exp-int-AF-intro-a} \\
\label{eq:exp-int-AF-intro-b}
V^{(k+1)} &= V^{(k)}+V\big(h_{k};\Omega^{(k)},W^{(k)}\big),\quad k=0,1,2,\dots
\intertext{can be computed in order to approximate \eqref{eq:AF-intro} more closely, where $V\big(h_{k};\Omega,W^{(k)}\big)$ solves the corresponding updated ODE \eqref{eq:LAF-intro} of the form
} \label{eq:exp-int-AF-intro-c}
\dot V &= A(\Omega^{(k)}; W^{(k)}) V + \Pi_{0}S(W^{(k)}).
\end{align}
\end{subequations}
The time-discrete equations \eqref{eq:exp-int-AF-intro} reveal two basic ingredients of deep networks (or neural ODEs) which the full assignment flow \eqref{eq:AF-intro} embodies in a continuous-time manner: coupling a pointwise nonlinearity \eqref{eq:exp-int-AF-intro-a} and diffusion \eqref{eq:exp-int-AF-intro-b},\eqref{eq:exp-int-AF-intro-c} enhances the expressivity of network models for data analysis.

The key point motivating the work reported in this paper is that our results apply to learning the parameters $\Omega^{k}$ in each step of the iterative scheme \eqref{eq:exp-int-AF-intro}.  We expect that the gradient, and its low-dimensional subspace representations, will help the further study of how each ingredient of \eqref{eq:exp-int-AF-intro} impacts the predictive power of assignment flows. Furthermore, `deep' extensions of \eqref{eq:AF-intro} and \eqref{eq:exp-int-AF-intro} are equally feasible within the \textit{same} mathematical framework (cf.~Section \ref{ssec:Application_Generalized_LAF}).

\subsection{Related Work}\label{sec:Related-Work}
\textit{Assignment flows} were introduced by \cite{Astrom:2017ac}. For a survey of prior and recent related work, we refer to \cite{Schnorr:2019aa}.
\textit{Linearized} assignment flows were introduced by \cite{Zeilmann2020} as part of a comprehensive study of numerical schemes for the geometric integration of the assignment flow equation \eqref{eq:AF-intro}.


While the bulk of these schemes are based on a Lie group action (cf.~\cite{Iserles:2000ab})
on the assignment manifold, which enables to apply established theory and algorithms for the numerical integration of ODEs that evolve in an Euclidean space \cite{Hairer:2008aa}, the linearity of the ODE \eqref{eq:LAF-intro} specifically allows to represent its solution in \textit{closed form}  by the Duhamel (or variation-of-constants) formula \cite{Teschl:2012aa}.
Corresponding extensions to \textit{nonlinear} ODEs rely on exponential integration \cite{Hochbruck:2009uk,Hochbruck:2010aa}. Iteration \eqref{eq:exp-int-AF-intro} combines a corresponding iterative scheme and the tangent-space based parametrization \eqref{eq:W-VT} of the linearized assignment flow.

A key computational step of the latter class of methods requires to evaluate an analytical matrix-valued function, like the matrix exponential and similar functions \cite[Section 10]{Higham2008}. While basic methods \cite{Moler2003} only work for problem of small and medium size, dedicated methods using Krylov subspaces \cite{Hochbruck:1997aa,Al-Mohy2011} and established numerical linear algebra \cite{Saad:1992aa,Saad2003} can be applied to larger problems. The algorithm that results from our approach employs such methods.

Machine learning requires to compute gradients of loss functions that take solutions of ODEs as argument. This defines an enormous computational task and explains why automatic differentiation and corresponding software tools are almost exclusively applied. Alternative dedicated recent methods like \cite{Kandolf2021} focus on a special problem structure, viz.~the action of the differential of the matrix exponential on a rank-one matrix.
Our closed form formula for the parameter gradient also involves the differential of a matrix exponential. Yet, we wish to evaluate the gradient itself rather than its action on another matrix. The special problem structure that we can exploit is the Kronecker sum of matrices. Accordingly, our approach is based on the recent corresponding work \cite{Benzi2017} and an additional subsequent low-rank approximation.

\subsection{Contribution, Organization}
We derive a closed form expression of the gradient of any $C^{1}$ loss function of the form \eqref{eq:loss-intro} that depends on the solution $V(t)$ of the linear system of ODEs \eqref{eq:LAF-intro} at some arbitrary but fixed time $t=T$. In addition, we develop a numerical method that enables to evaluate the gradient efficiently for the common large sizes of image labeling problems. We apply the method to optimal parameter estimation by Riemannian gradient descent and validate our approach by a series of proof-of-concept experiments. This includes a comparison with automatic differentiation applied to two numerical schemes for integrating the linearized assignment flow: geometric explicit Euler and exponential integration. It turns out that our method is as accurate and efficient as the highly optimized automatic differentiation software, like PyTorch~\cite{Paszke2019} and TensorFlow~\cite{Abadi2015}. We point out that to our knowledge, automatic differentiation has not been applied to exponential integration, so far.

This paper extends the conference paper \cite{Zeilmann:2021wt} in that all parameter dependencies of the loss function, constrained by the linearized assignment flow, are taken into account (cf.~diagram \eqref{eq:Omega-dependency}). In addition, a complete proof of the corresponding main result (Theorem \ref{thm:loss-function-gradient}) is provided. The space complexity of various gradient approximations are specified in a series of Remarks. The approach is validated numerically and more comprehensively by comparing to automatic differentiation and by examining the influence of all parameters.

The plan for this paper is as follows.
Section~\ref{sec:Preliminaries} summarizes
the assignment flow approach, the linearized assignment flow and exponential integration
for integrating the latter flow.
Section~\ref{sec:GradientApproximation} details the derivation of the exact gradient of
any loss function of the flow with respect to the weight parameters that regularize the flow. Furthermore, a low-rank approximation of the gradient is developed for evaluating the gradient efficiently. We also sketch how automatic derivation is applied to two numerical schemes in order to solve the parameter estimation problem in alternative ways. Numerical experiments are reported in Section \ref{sec:Experiments} for comparing the methods and for inspecting quantitatively the gradient approximation and properties of the estimated weight patches that parametrize the linearized assignment flow.
We conclude in Section \ref{sec:Conclusion} and point out further directions of research.

\section{Preliminaries}
\label{sec:Preliminaries}



\subsection{Basic Notation}
\label{sec:Basic-Notation}
We set $[n]=\{1,2,\dotsc,n\}$ for $n\in\N$. The cardinality of a finite set $S$ is denoted by $|S|$, e.g.~$|[n]|=n$. $\R^{n}_{+}$ denotes the positive orthant and $\R_{>}^{n}$ its interior. $\eins=(1,1,\dotsc,1)^{\T}$ has dimension depending on the context that we specify sometimes by a subscript, e.g.~$\eins_{n}\in\R^{n}$. Similarly, we set $0_{n}=(0,0,\dotsc,0)^{\T}\in\R^{n}$. $\{e_{i}\colon i\in[n]\}$ is the canonical basis of $\R^{n}$ and $I_{n}=(e_{1},\dotsc,e_{n})\in\R^{n\times n}$ the identity matrix.

The support of a vector $x\in \R^{n}$ is denoted by $\supp(x) = \{i\in[n]\colon x_{i}\neq 0\}$. $\Delta_{n}=\{p\in\R_{+}^{n}\colon\la\eins_{n},p\ra=1\}$ is the probability simplex whose points represent discrete distributions on $[n]$. Distributions with full support $[n]$ form the relative interior $\mathring\Delta_{n}=\Delta_{n}\cap\R_{>}^{n}$. $\la\cdot,\cdot\ra$ is the Euclidean inner product of vectors and matrices. In the latter case, this reads $\la A, B\ra = \tr(A^{\T} B)$ with the trace $\tr(A)=\sum_{i}A_{ii}$. The induced Frobenius norm is denoted by $\|A\|=\sqrt{\la A,A,\ra}$, and other matrix norms like the spectral norm $\|A\|_{2}$  are indicated by subscripts. The mapping $\Diag\colon\R^{n}\to\R^{n\times n}$ sends a vector $x$ to the diagonal matrix $\Diag(x)$ with entries $x$. $A\otimes B$ denotes the Kronecker \textit{product} of matrices $A$ and $B$ \cite{Graham:1981wj,Van-Loan:2000aa} and $\oplus$ the Kronecker \textit{sum}
\begin{equation}\label{eq:def-Kronecker-sum}
    A \oplus B = A \otimes I_{n} + I_{m} \otimes B \in \R^{m n\times m n},\qquad A \in \R^{m\times m},\quad B\in\R^{n\times n}.
\end{equation}
We have
\begin{equation}\label{eq:otimes-mixed}
    (A\otimes B)(C\otimes D) = (A C)\otimes (B D)
\end{equation}
for matrices of compatible dimensions.
The operator $\vvec_{r}$ turns a matrix into the vector by stacking the row vectors. It satisfies
\begin{equation}\label{eq:def-vecr}
    \vvec_{r}(A B C) = (A \otimes C^{\T})\vvec_{r}(B).
\end{equation}
The Kronecker product $v \otimes w \in \R^{mn}$ of two vectors $v \in \R^m$ and
$w \in \R^n$ is defined by viewing the vectors as matrices with only one column
and applying the definition of Kronecker products for matrices. We have
\begin{equation}\label{eq:otimes-vecr}
    v \otimes w = \vvec_{r}(v w^\top).
\end{equation}
The matrix exponential of a square matrix $A$ is given by \cite[Ch.~10]{Higham2008}
\begin{equation}
    \expm(A) = \sum_{k\geq 0} \frac{A^{k}}{k!}.
\end{equation}
$L(\mc{E}_{1},\mc{E}_{2})$ denotes the space of all linear bounded mappings from $\mc{E}_{1}$ to $\mc{E}_{2}$.

\subsection{Assignment Flow}\label{ssec:Assignment_Flow}

Let $G=(I,E)$ be a given undirected graph with vertices $i \in I$ indexing data
\begin{equation}\label{eq:mcF_I}
    \mc{F}_{I} = \{f_{i} \colon i \in I\} \subset \mc{F}
\end{equation}
given in a metric space $(\mc{F},d)$.
In this paper, we focus primarily on the application of image labeling in which the graph $G$ is a grid graph equipped with a $3 \times 3$ or larger neighborhood $\mc{N}_{i} = \{k \in I \colon ik=ki \in E\} \cup \{i\}$ at each pixel $ i\in I$.
The linearized assignment flow and the learning approach in this paper can, however, also be applied to the case of data labeling on arbitrary graphs.

Along with $\mc{F}_{I}$, \textit{prototypical data (labels)}
$\mc{L}_{J} = \{l_{j} \in \mc{F} \colon j \in J\}$ are given that represent
classes $j = 1,\dotsc,|J|$.
\textit{Supervised image labeling} denotes the task to assign precisely one
prototype $l_{j}$ to each datum $f_{i}$ at every vertex $i$ in a coherent way,
depending on the label assignments in the neighborhoods $\mc{N}_{i}$.
These assignments at $i$ are represented by probability vectors
\begin{equation}\label{eq:def-Wi}
    W_{i} \in \mathring{\Delta}_{|J|},\quad
    i \in I.
\end{equation}
The set $\mathring{\Delta}_{|J|}$ becomes a
Riemannian manifold denoted by $\mc{S} := (\mathring\Delta_{|J|},g_{\sst{FR}})$ when endowed with the
Fisher-Rao metric $g_{\sst{FR}}$.
Collecting all assignment vectors as \textit{rows} defines the strictly positive
row-stochastic \textit{assignment matrix}
\begin{equation}\label{eq:def-mcW}
    W = {(W_{1},\dotsc,W_{|I|})}^{\T} \in \mc{W} = \mc{S} \times \dots \times \mc{S} \subset \R^{|I| \times |J|},
\end{equation}
that we regard as point on the product \textit{assignment manifold} $\mc{W}$.
Image labeling is accomplished by geometrically integrating the
\textit{assignment flow} $W(t)$ solving
\begin{equation}\label{eq:def-dot-W}
    \dot W = R_{W}\big(S(W)\big),\qquad
    W(0) = \eins_{\mc{W}} := \frac{1}{|J|} \eins_{|I|} \eins_{|J|}^{\T}\qquad (\text{barycenter}),
\end{equation}
where $R_{W}$ and $S(W)$ are defined in \eqref{eq:def-R-p} resp.~\eqref{eq:def-Si}.
The assignment flow provably converges towards a binary matrix~\cite{Zern:2020aa},
i.e.~$\lim_{t\to\infty}W_{i}(t)=e_{j(i)}$, for every $i\in I$ and some $j(i)\in J$,
which yields the label assignment $f_{i} \mapsto l_{j(i)}$.
In practice, geometric integration is terminated when $W(t)$ is $\veps$-close to
an integral point using the entropy criterion from~\cite{Astrom:2017ac}, followed
by trivial rounding, due to the existence of basins of attraction around each integral point  \cite{Zern:2020aa}.

We specify the right-hand side of the differential equation in~\eqref{eq:def-dot-W} --- see \eqref{eq:RW-SW} and \eqref{eq:def-Si}
below --- and refer to~\cite{Astrom:2017ac,Schnorr:2019aa} for more details and the background.
With the tangent space
\begin{equation}\label{eq:def-T0}
    T_{0}=T_{p}\mc{S} = \{v\in\R^{|J|}\colon \la \eins,v\ra=0\},\qquad\forall p\in\mc{S},
\end{equation}
that does not depend on the base point
$p \in \mc{S}$, we define
\begin{subequations}\allowdisplaybreaks
    \begin{alignat}{2}
    \Pi_{0} \colon \R^{|J|} &\rightarrow T_{0}, &\quad z &\mapsto I_{|J|}-\frac{1}{|J|}\eins_{|J|}\eins_{|J|}^{\T}, \label{eq:def-Pi0} \\
        R_{p} \colon \R^{|J|} &\rightarrow T_{0}, &\quad z &\mapsto R_{p}(z)=\big(\Diag(p)-p p^{\T}\big) z,
        \label{eq:def-R-p} \\
        \Exp \colon \mc{S} \times T_{0} &\rightarrow \mc{S}, &\quad (p,v) &\mapsto
        \Exp_{p}(v) = \frac{e^{\frac{v}{p}}}{\la p, e^{\frac{v}{p}} \ra} p,
        \\
        \label{eq:def-Exp-inv}
        \Exp^{-1} \colon \mc{S} \times \mc{S} &\rightarrow T_{0}, &\quad (p,q) &\mapsto \Exp_{p}^{-1}(q)
        = R_{p}\log\frac{q}{p},
        \\ \label{eq:def-exp}
        \exp \colon \mc{S} \times \R^{|J|} &\rightarrow \mc{S}, &\quad (p,z) &\mapsto \exp_{p}(z) = \Exp_{p}\circ R_{p}(z) = \frac{p e^{z}}{\la p,e^{z}\ra},
    \end{alignat}
\end{subequations}
where multiplication, division, exponentiation $e^{(\cdot)}$ and $\log(\cdot)$
apply \textit{component-wise} to vectors.
Corresponding maps
\begin{equation}\label{eq:def-R-Exp-exp-product}
    R_{W}, \qquad \Exp_{W}, \qquad \exp_{W}
\end{equation}
in connection with the product
manifold~\eqref{eq:def-mcW} are defined analogously, and likewise the tangent
space
\begin{equation}\label{eq:def-mcT-0}
    \mc{T}_{0}=T_{0} \times \dots \times T_{0} = T_{W}\mc{W},\qquad \forall W\in\mc{W}
\end{equation}
and the extension of the orthogonal projection \eqref{eq:def-Pi0} onto $\mc{T}_{0}$, again denoted by $\Pi_{0}$.
For example, regarding \eqref{eq:def-dot-W}, with $W \in \mc{W}$ and $S(W)\in\mc{W}$ (or more generally $S \in \R^{|I|\times |J|}$), we have
\begin{subequations}\label{eq:RW-SW}
\begin{align}
R_{W}S(W)
&= \big(R_{W_{1}} S_{1}(W),\dotsc, R_{W_{|I|}} S_{|I|}(W)\big)^{\T}
= \vvec_{r}^{-1}\big(\Diag(R_{W})\vvec_{r}\big(S(W)\big)\big)
\intertext{with}\label{eq:def-Diag-R}
\Diag(R_{W}) &:= \bpm
R_{W_{1}} & 0 & \dotsb & 0 \\
0 & R_{W_{2}} & & \vdots \\
\vdots & & \ddots & 0 \\
0 & \dotsb & & R_{W_{|I|}}
\epm.
\end{align}
\end{subequations}

Given data $\mc{F}_{I}$ are taken into account as distance vectors
\begin{equation}\label{eq:def-Di}
D_{i}=\big(d(f_{i},l_{1}),\dotsc,d(f_{i},l_{|J|})\big)^{\T},\quad i\in I
\end{equation}
and mapped to $\mc{W}$ by
\begin{equation}\label{eq:def-Li}
   L(W) = \exp_{W}(-\tfrac{1}{\rho}D) \in \mc{W},\quad
   L_{i}(W_{i}) = \exp_{W_{i}}(-\tfrac{1}{\rho}D_{i}) = \frac{W_{i}e^{-\frac{1}{\rho} D_{i}}}{\la W_{i},e^{-\frac{1}{\rho} D_{i}}\ra},
\end{equation}
where $\rho > 0$ is a user parameter for normalizing the scale of the data.
These \textit{likelihood vectors} represent data terms in conventional
variational approaches: Each individual flow $\dot W_{i} = R_{W_{i}} L_{i}(W_{i})$,
$W_{i}(0)=\eins_{\mc{S}}$ converges to $e_{j(i)}$ with $j(i)=\arg\min_{j\in J} D_{ij}$
and in this sense maximizes the local data likelihood.

The vector field defining the assignment flow~\eqref{eq:def-dot-W} arises through
\textit{coupling} flows for individual pixels through \textit{geometric averaging}
within the neighborhoods $\mc{N}_{i},\,i\in I$, conforming to the underlying
Fisher-Rao geometry
\begin{subequations}\label{eq:def-Si}
    \begin{align}
        S(W) &= \bpm \vdots \\ {S_{i}(W)}^{\T} \\ \vdots \epm
        = \mathcal{G}^{\Omega}\big(L(W)\big) \in \mc{W},\qquad
        \label{eq:def-Si-a} \\ \label{eq:def-Si-b}
        S_{i}(W)
        &= \mathcal{G}^{\Omega}_{i}\big(L(W)\big) = \Exp_{W_{i}} \Big(\sum_{k \in \mc{N}_{i}} \w_{ik} \Exp_{W_{i}}^{-1}\big(L_{k}(W_{k})\big)\Big),\quad i \in I.
    \end{align}
\end{subequations}
The \textit{similarity vectors} $S_{i}(W)$ are parametrized by
strictly positive \textit{weight patches} $(\w_{ik})_{k\in\mc{N}_{i}}$, centered at $i\in I$ and indexed by local neighborhoods $\mc{N}_{i}\subset I$, that in turn define the \textit{weight parameter matrix}
\begin{equation}\label{eq:def-Omega}
    \Omega = {(\Omega_{i})}_{i\in I} \in \R_{+}^{|I|\times |I|},\qquad
    \Omega_{i}|_{\mc{N}_{i}} = {(\w_{ik})}_{k\in\mc{N}_{i}} \in \mathring\Delta_{|\mc{N}_{i}|},\qquad
    \sum_{k\in\mc{N}_{i}}\w_{ik}=1,\;\forall i\in I.
\end{equation}
The matrix $\Omega$ comprises all \textit{regularization parameters} satisfying the latter linear constraints.
Flattening these weight patches defines row vectors $\Omega_{i}|_{\mc{N}_{i}},\,i\in I$ and, by complementing with $0$, entries of the \textit{sparse} row vectors $\Omega_{i}$
of the matrix $\Omega$. Note that the positivity assumption $\w_{ik}>0$ is reflected by the membership $\Omega_{i}|_{\mc{N}_{i}} \in \mathring\Delta_{|\mc{N}_{i}|}$. Throughout this paper, we assume that all pixels have neighborhoods of equal size
\begin{equation}\label{eq:ass-mcN}
    |\mc{N}| := |\mc{N}_{i}|,\quad\forall i\in I
\end{equation}
and therefore simply write $\Omega_{i}|_{\mc{N}} = \Omega_{i}|_{\mc{N}_{i}}$.
These parameters are used in the linearized assignment flow, to be introduced next.
We explain a corresponding parameter estimation approach in Section \ref{sec:GradientApproximation}
and a parameter predictor in Section~\ref{ssec:Parameter_Prediction}.

\subsection{Linearized Assignment Flow}\label{sec:LAF}

The \textit{linearized assignment flow}, introduced by~\cite{Zeilmann2020},
approximates~\eqref{eq:def-dot-W} by
\begin{equation}\label{eq:LAF}
    \dot W = R_{W}\Big(S(W_{0}) + dS_{W_{0}}R_{W_{0}} \log\frac{W}{W_{0}}\Big),
    \quad W(0)=W_{0} \in \mc{W}
\end{equation}
around any point $W_{0}$.
In what follows, we only consider the barycenter \begin{equation}\label{eq:def-W0}
W_{0}=\eins_{\mc{W}}
\end{equation}
which is the initial point of~\eqref{eq:def-dot-W}.
The differential equation~\eqref{eq:LAF} is still \textit{nonlinear} but can be
parametrized by a \textit{linear} ODE on the tangent space
\begin{subequations}\label{eq:LAF-V}
    \begin{align}
        W(t) &= \Exp_{W_{0}}\big(V(t)\big),
        \label{eq:LAF-V-a} \\ \label{eq:LAF-V-b}
        \dot V &= R_{W_{0}}\big(S(W_{0}) + dS_{W_{0}} V\big)
        =: B_{W_0} + A(\Omega)V,\quad V(0)=0,
    \end{align}
\end{subequations}
where matrix $A(\Omega)$ linearly depends on the parameters $\Omega$
of~\eqref{eq:def-Si}. The action of $A(\Omega)$ on $V$ is explicitly given by \cite[Prop.~4.4]{Zeilmann2020}
\begin{subequations}\label{eq:A-Omega}
\begin{align}
A(\Omega) V
&=R_{W_{0}}dS_{W_{0}}V
= R_{S(W_{0})}\Omega V
\overset{\eqref{eq:RW-SW}}{=}
\vvec_{r}^{-1}\big(\Diag(R_{S(W_{0})})\vvec_{r}(\Omega V)\big)
\\
&= \bigg(R_{S_{1}(W_{0})}\sum_{k\in\mc{N}_{1}}\w_{1k}V_{k},\dotsc,R_{S_{|I|}(W_{0})}\sum_{k\in\mc{N}_{|I|}}\w_{|I|k}V_{k}\bigg)^{\T},
\end{align}
\end{subequations}
where $\Diag(R_{S(W_{0})})$ is defined by \eqref{eq:def-Diag-R} and we took into account \eqref{eq:def-W0}.
The linear ODE~\eqref{eq:LAF-V-b} admits a closed-form solution which in turn
enables a different numerical approach (Section~\ref{ssec:Exponential_Integration})
and a novel approach to parameter learning (Section~\ref{sec:GradientApproximation}).

\subsection{Exponential Integration}\label{ssec:Exponential_Integration}
The solution to~\eqref{eq:LAF-V-b} is given by a high-dimensional integral (Duhamel's formula) whose  value in closed form is given by
\begin{equation}\label{eq:V-t-vphi}
    V(t;\Omega) = t \varphi\big(tA(\Omega)\big) B_{W_0},
    \qquad
    \vphi(x) = \frac{e^{x}-1}{x}
    = \sum_{k=0}^{\infty} \frac{x^{k}}{(k+1)!},
\end{equation}
where the entire function $\varphi$
is extended to matrix arguments as the limit of an absolutely convergent power series in the matrix space \cite[Theorem 6.2.8]{Horn:1991tx}.
As the matrix $A$ is already very large even for medium-sized images, however,
it is not feasible in practice to compute $\varphi(tA)$ in this way.
Exponential integration~\cite{Hochbruck:1997aa,Niesen2012}, therefore, was used
by \cite{Zeilmann2020} for approximately evaluating~\eqref{eq:V-t-vphi}, as sketched next.

Applying the row-stacking operator~\eqref{eq:def-vecr}
to both sides of~\eqref{eq:LAF-V-b}
and~\eqref{eq:V-t-vphi}, respectively, yields with
\begin{equation}
v = \vvec_{r}(V)
\end{equation}
the ODE~\eqref{eq:LAF-V-b} in the form
\begin{subequations}\label{eq:vvec-v}\allowdisplaybreaks
    \begin{align}
        \label{eq:vvec-v-a}
        \dot v &= b + A^{J}(\Omega) v,& v(0)&=0, \qquad b = b(\Omega) = \vvec_{r}(B_{W_0}) \in \R^{n}, \\
        \label{eq:vvec-v-b}
        A^J(\Omega) &= {\big(A^J_{ik}(\Omega)\big)}_{i,k \in I} \in \R^{n\times n},\quad
        &A^J_{ik}(\Omega) &= \begin{cases}
            \w_{ik} R_{S_{i}(W_{0})},
            & k \in \mc{N}_{i}, \\
            0, & k \not\in\mc{N}_{i}.
        \end{cases} \\
        \label{eq:vvec-v-c}
        v(t;\Omega)
        &= t\vphi\big(t A^{J}(\Omega)\big) b, &
        n &:= \dim v(t;\Omega) = |I| |J|,
    \end{align}
\end{subequations}
where $A^{J}(\Omega)$ results from
\begin{subequations}\label{eq:AJ-Omega}
\begin{align}
\vvec_{r}\big(A(\Omega) V\big)
\overset{\eqref{eq:A-Omega}}&{=}
\Diag(R_{S(W_{0})})\vvec_{r}(\Omega V)
= \Diag(R_{S(W_{0})})(\Omega\otimes I_{|J|}) v
\\
&= A^{J}(\Omega) v.
\end{align}
\end{subequations}

Using the Arnoldi iteration~\cite{Saad2003} with initial vector $q_{1}=b/\|b\|$,
we determine an orthonormal basis
$Q_{m}=(q_{1},\dotsc,q_{m}) \in \Reals^{n\times m}$ of the Krylov space
$\mathcal{K}_m(A^{J}, b)$ of dimension $m$.
As will be validated in Section \ref{sec:Experiments}, choosing $m\leq 10$ yields sufficiently
accurate approximations of the actions of the matrix exponential $\expm$ and the
$\vphi$ operator on a vector, respectively, that are given by
\begin{subequations}
    \begin{align}
        \expm\big(tA^{J}(\Omega)\big)b
        &\approx \|b\| Q_{m} \expm(t H_{m})e_1,\qquad
        H_{m}=Q_{m}^{\T}A^{J}(\Omega)Q_{m},
        \label{eq:Krylov-expm} \\ \label{eq:Krylov-vphi}
        t \varphi\big(tA^{J}(\Omega)\big)b
        &\approx t \|b\| Q_{m} \varphi(t H_{m})e_1.
    \end{align}
\end{subequations}
The expression $\varphi(t H_{m})e_1$ results from computing the left-hand side
of the relation~\cite[Section 10.7.4]{Higham2008}
\begin{equation}\label{eq:expm-0}
    \expm\bpm
    t H_{m} & e_{1} \\ 0 & 0
    \epm
    = \bpm
    \expm(t H_{m}) & \vphi(t H_{m}) e_{1}
    \\ 0 & 1
    \epm
\end{equation}
and extracting the upper-right vector.
Since $H_{m}$ is a small matrix, any standard method~\cite{Moler2003} can be
used for computing the matrix exponential on the left-hand side.

\section{Parameter Estimation}
\label{sec:GradientApproximation}

Section \ref{sec:Learning-Procedure} details our approach for learning optimal weight parameters
for a given image and ground truth labeling:  Riemannian gradient descent is performed with respect to a loss function that depends on the solution of the linearized assignment flow. A closed form expression of this gradient is derived in Section \ref{sec:Loss-Function-Gradient} along with a low-rank approximation in Section \ref{sec:Gradient-Approximation} that can be computed efficiently. As an alternative and baseline, we outline in Section \ref{ssec:AutoDiff} two gradient approximations based on numerical schemes for integrating the linearized assignment flow and automatic differentiation.

\subsection{Learning Procedure}\label{sec:Learning-Procedure}
Let
\begin{equation}\label{eq:def-P-Omega}
    P_{\Omega} = \{\Omega \in \R_{+}^{|I|\times |I|}\colon \Omega\;\text{satisfies~\eqref{eq:def-Omega}}\}
\end{equation}
denote the space of weight parameter matrices that parametrize the similarity
mapping~\eqref{eq:def-Si}.
Due to~\eqref{eq:def-Omega} and~\eqref{eq:ass-mcN}, the restrictions
$\Omega_{i}|_{\mc{N}}$ are strictly positive probability vectors, as are the
assignment vectors $W_{i}$ defined by~\eqref{eq:def-Wi}.
Therefore, similar to $W_{i}\in\mc{S}$, we consider each $\Omega_{i}|_{\mc{N}}$ as
point on a corresponding manifold $(\Delta_{|\mc{N}|},g_{\sst{FR}})$ equipped
with the Fisher-Rao metric and --- in this sense --- regard $P_{\Omega}$ in \eqref{eq:def-P-Omega} as
corresponding product manifold.

Let $W^{\ast}\in\mc{W}$ denote the ground truth labeling for a given image, and
let $V^{\ast} = \Pi_0 W^{\ast} \in\mc{T}_{0}$ be a tangent vector such that
$\lim_{s\to\infty} \Exp_{\eins_{\mc{W}}}(s V^{\ast}) = W^{\ast}$.
Our objective is to determine $\Omega$ such that, for some specified time $T>0$, the vector
\begin{equation}\label{eq:def-V_T}
    V_{T}(\Omega) := V(T;\Omega),
\end{equation}
given by~\eqref{eq:V-t-vphi} and corresponding to the linearized assignment flow,
approximates the \textit{direction} of $V^{\ast}$ and hence
\begin{equation}\label{eq:from-V-to-W}
    \lim_{s\to\infty}\Exp_{\eins_{\mc{W}}}\big(s V_{T}(\Omega)\big) = W^{\ast}.
\end{equation}
In this formula the direction of the vector $V_{T}(\Omega)$ only is relevant, but not its magnitude.
A distance function that also satisfies these properties is given by
\begin{equation}\label{eq:cosineSimilarity}
    f_{\mc{L}}\colon\mc{T}_{0}\to\R,\qquad
    V\mapsto 1-\frac{\la V^{\ast},V\ra}{\|V^{\ast}\|\|V\|}.
\end{equation}
In addition, we consider a regularizer
\begin{equation}\label{eq:R-Omega}
    \mc{R}\colon P_{\Omega} \to \R,\qquad
    \Omega\mapsto \frac{\tau}{2} \sum_{i\in I}\|t_{i}(\Omega)\|^{2},\qquad
    t_{i}(\Omega) = \exp_{\eins_{\Omega}}^{-1}(\Omega_{i}|_{\mc{N}}),\qquad \tau > 0
\end{equation}
and define the loss function
\begin{equation}\label{eq:Omega-loss}
    \mc{L}\colon P_{\Omega}\to\R,\qquad
    \mc{L}(\Omega) = f_{\mc{L}}\big(V_{T}(\Omega)\big) + \mc{R}(\Omega),
\end{equation}
with $V_{T}(\Omega)$ from \eqref{eq:def-V_T}. $\Omega$ is determined by the Riemannian gradient descent sequence
\begin{equation}\label{eq:R-descent-iteration}
    \Omega^{(k+1)} = \exp_{\Omega^{(k)}}\big(-h\nabla\mc{L}(\Omega^{(k)})\big),\quad
    k\geq 0,\qquad \Omega^{(0)}_{i}|_{\mc{N}} = \eins_{|\mc{N}|},\quad i\in I
\end{equation}
with step size $h>0$. Here
\begin{equation}\label{eq:R-grad-mcL}
\nabla\mc{L}(\Omega)=R_{\Omega}\partial\mc{L}(\Omega)
\end{equation}
is the Riemannian gradient with respect to the Fisher-Rao metric.
$R_{\Omega}$ is given by~\eqref{eq:def-R-Exp-exp-product} and~\eqref{eq:def-R-p}
and effectively applies to the restrictions $\Omega_{i}|_{\mc{N}}$ of the row
vectors with all remaining components equal to $0$.
It remains to compute the Euclidean gradient $\partial\mc{L}(\Omega)$ of the loss function \eqref{eq:Omega-loss} which is presented in the subsequent Section \ref{sec:Loss-Function-Gradient}.


\subsection{Loss Function Gradient}\label{sec:Loss-Function-Gradient}

In Section \ref{sec:Closed-Form-Gradient} we derive a closed form expression for the loss function
gradient (Theorem \ref{thm:loss-function-gradient}), after introducing some basic calculus rules for representing and computing
differentials of matrix-valued mappings in Section~\ref{sec:Matrix-Differentials}.

\subsubsection{Matrix Differentials}\label{sec:Matrix-Differentials}
Let $F \colon \R^{m_{1}\times m_{2}} \to \R^{n_{1}\times n_{2}}$ be a smooth mapping.
Using the canonical identification $T\mc{E} \cong \mc{E}$ of the tangent spaces
of any Euclidean space $\mc{E}$ with $\mc{E}$ itself, we both represent and
compute the differential
\begin{equation}\label{eq:dF}
    dF\colon \R^{m_{1}\times m_{2}} \to L(\R^{m_{1}\times m_{2}},\R^{n_{1}\times n_{2}})
\end{equation}
in terms of a vector-valued mapping $f$, which is defined by $F$ according to the commutative diagram
\begin{equation}
    \begin{tikzcd}
        {\Reals^{m_{1} \times m_{2}}} & & {\Reals^{n_{1} \times n_{2}}} \\
        & \substack{L(\R^{m_{1}\times m_{2}},\; \R^{n_{1}\times n_{2}}) \\ \cong\; \R^{n_{1}n_{2}\times m_{1}m_{2}}} \\
        {\Reals^{m_{1}m_{2}}} & & {\Reals^{n_{1}n_{2}}}
        \arrow["F", from=1-1, to=1-3]
        \arrow["f", from=3-1, to=3-3]
        \arrow["{\vvec_r}", from=1-1, to=3-1]
        \arrow["{\vvec_r}"', from=1-3, to=3-3]
        \arrow["dF", from=1-1, to=2-2]
        \arrow["df", from=3-1, to=2-2]
    \end{tikzcd}
\end{equation}
In formulas, this means that based on the equation
\begin{equation}\label{eq:vecr-F-f-vecr}
    \vvec_{r}\big(F(X)\big)=f\big(\vvec_{r}(X)\big),\quad\forall  X\in\R^{m_{1}\times m_{2}},
\end{equation}
we set
\begin{equation}\label{eq:action-dF-df}
    \vvec_{r}\big(dF(X)Y)
    = df\big(\vvec_{r}(X)\big)\vvec_{r}(Y),\qquad
    \forall X, Y \in \R^{m_{1}\times m_{2}}
\end{equation}
and hence \textit{define} and compute the differential~\eqref{eq:dF} as matrix-valued mapping
\begin{equation}\label{eq:dF-by-df}
    dF := df \circ \vvec_{r}.
\end{equation}
The corresponding linear actions on $Y\in\R^{m_{1}\times m_{2}}$ and
$\vvec_{r}(Y)\in \R^{m_{1}m_{2}}$, respectively, are given by~\eqref{eq:action-dF-df}.
We state an auxiliary result required in the next subsection, which also provides a first concrete instance of the general relation \eqref{eq:action-dF-df}.
\begin{lemma}[\textbf{differential of the matrix exponential}]\label{lem:dexpm}
    If $F = \expm\colon\R^{n\times n}\to\R^{n\times n}$, then~\eqref{eq:action-dF-df} reads
    \begin{equation}
        \vvec_{r}\big(d\expm(X) Y\big)
        = \big(\expm(X)\otimes I_{n}\big)\vphi(-X\oplus X^{\T})\vvec_{r}(Y),\quad Y\in\R^{n\times n},
    \end{equation}
    with $\vphi$ given by~\eqref{eq:V-t-vphi}.
\end{lemma}
\begin{proof}
    The result follows from~\cite[Thm. 10.13]{Higham2008} where columnwise
    vectorization is used, after rearranging so as to conform to the row-stacking
    mapping $\vvec_{r}$ used in this paper.
\end{proof}

\subsubsection{Closed-Form Gradient Expression}\label{sec:Closed-Form-Gradient}

We separate the computation of $\mathcal{L}(\Omega)$ and the gradient $\partial \mathcal{L}(\Omega)$ into several operations that were introduced in sections
\ref{sec:Preliminaries} and \ref{sec:Learning-Procedure}.
We illustrate their composition and accordingly the process from parameters $\Omega$ to a loss $\mc{L}(\Omega)$ in the following flow diagram that refers to quantities in \eqref{eq:vvec-v} and \eqref{eq:AJ-Omega} related to the linearized assignment flow, after vectorization.
\begin{equation}\label{eq:Omega-dependency}
    \begin{tikzcd}
        \Omega & {S(W_0) = \exp_{\eins_\mathcal{W}}\left(-\frac{1}{\rho} \Omega D\right)} & {b(\Omega)=\vvec_r(R_{W_0}S(W_0))} \\
        & {A^J(\Omega)=\Diag(R_{S(W_0)})(\Omega \otimes I_{|J|})} & {v_{T}(\Omega)=T\varphi\left(TA^J(\Omega)\right)b(\Omega)} \\
        & {\mc{R}(\Omega)} & {\mc{L}(\Omega) = f_{\mathcal{L}}(v_{T}(\Omega))+\mc{R}(\Omega)}
        \arrow["\text{(M1)}", from=1-1, to=1-2]
        \arrow["\text{(M2)}", from=1-2, to=1-3]
        \arrow["\text{(M3)}"{pos=0.7}, from=1-1, to=2-2]
        \arrow[from=1-2, to=2-2]
        \arrow["\text{(M4)}", from=1-3, to=2-3]
        \arrow["\text{(M4)}", from=2-2, to=2-3]
        \arrow[from=2-3, to=3-3]
        \arrow["\text{(M5)}", swap, from=1-1, to=3-2,
            start anchor={south}, out=270,
            end anchor={west}, in=180
        ]
        \arrow[from=3-2, to=3-3]
    \end{tikzcd}
\end{equation}
In what follows, we traverse this diagram from top-left to bottom-right and collect each partial result by a corresponding lemma or proposition. Theorem \ref{thm:loss-function-gradient} assembles all results and provides a closed form expression of the loss function gradient $\partial\mc{L}(\Omega)$. To enhance readability, the proofs of most Lemmata are listed in Appendix \ref{app:closed-form-gradient}.

We focus on mapping (M1) in diagram \eqref{eq:Omega-dependency}.
\begin{lemma}\label{lem:A1}
The differential of the function
\begin{equation}\label{eq:f-A1}
f_{1} \colon \R^{|I| \times |I|} \to \R^{|I| \times |J|},\qquad
\Omega \mapsto f_{1}(\Omega)
:= S(W_{0}) = \exp_{\eins_{\mc{W}}}\Big(-\frac{1}{\rho} \Omega D\Big),\qquad
D\in\R^{|I|\times |J|}
\end{equation}
and its transpose are given by
\begin{subequations}\label{eq:df-A1}
\begin{align}
df_{1}(\Omega) Y
&= -\frac{1}{\rho} R_{f_{1}(\Omega)} (Y D),\qquad \forall Y\in\R^{|I|\times |I|},
\label{eq:df-A1-a} \\ \label{eq:df-A1-b}
df_{1}(\Omega)^{\T} Z
&= -\frac{1}{\rho}R_{f_{1}(\Omega)}(Z) D^{\T},\qquad \forall
Z\in\R^{|I|\times |J|},
\end{align}
\end{subequations}
with $R_{f_{1}(\Omega)}$ defined by \eqref{eq:RW-SW}.
\end{lemma}
\noindent
\textit{Proof:} see Appendix \ref{app:closed-form-gradient}. \\[0.1cm]

We consider mapping (M2) of diagram \eqref{eq:Omega-dependency}, taking into account mapping (M4) and notation \eqref{eq:f-A1}.
\begin{lemma}\label{lem:A2}
The differential of the function
\begin{equation}\label{eq:f-A2}
f_{2}\colon\R^{|I|\times |I|} \to \R^{|I|^{2}},\qquad
\Omega\mapsto f_{2}(\Omega)
:= b(\Omega) = \vvec_{r}\big(R_{W_{0}}f_{1}(\Omega)\big)
\end{equation}
and its transpose are given by
\begin{subequations}\label{eq:df-A2}
\begin{align}
df_{2}(\Omega) Y &= \vvec_{r}\big(R_{W_{0}}df_{1}(\Omega)Y\big),\qquad\forall Y\in\R^{|I|\times |I|}
\label{eq:df-A2-a} \\ \label{eq:df-A2-b}
df_{2}(\Omega)^{\T} Z
&= df_{1}(\Omega)^{\T}(R_{W_{0}} Z),\qquad\qquad
\forall Z\in\R^{|I|\times |I|}.
\end{align}
\end{subequations}
\end{lemma}
\noindent
\textit{Proof:} see Appendix \ref{app:closed-form-gradient}. \\[0.1cm]
We note that $d f_{2}(\Omega)^{\T}$ should act on a vector $\vvec_{r}(Z)\in\R^{|I|^{2}}$. We prefer the more compact and equivalent non-vectorized expression \eqref{eq:df-A2-b}.

We turn to mapping (M3) of diagram \eqref{eq:Omega-dependency} and use \eqref{eq:Omega-dependency}.
\begin{lemma}\label{lem:A3}
The differential of the mapping
\begin{equation}
f_{3}\colon\R^{|I|\times |I|}\to\R^{n\times n},\qquad
\Omega\mapsto f_{3}(\Omega)
:= A^{J}(\Omega)
= \Diag(R_{f_{1}(\Omega)})(\Omega\otimes I_{|J|}),\qquad n=|I| |J|
\end{equation}
is given by
\begin{subequations}\label{eq:df-A3}
\begin{align}
df_{3}(\Omega) Y
&= \Diag(d R_{f_{1}(\Omega)} Y)(\Omega\otimes I_{|J|})
+ \Diag(R_{f_{1}(\Omega)})(Y\otimes I_{|J|}),\qquad\forall Y\in\R^{|I|\times |I|}.
\intertext{
Here, $\Diag(d R_{f_{1}(\Omega)} Y)\in\R^{n\times n}$ is defined by \eqref{eq:def-Diag-R} and $|I|$ block matrices of size $|J|\times |J|$ on the diagonal of the form
}
d R_{f_{1i}(\Omega)} Y
&= \Diag\big(d f_{1i}(\Omega) Y\big)-\big(df_{1i}(\Omega) Y\big)f_{1i}(\Omega)^{\T} - f_{1i}(\Omega) \big(df_{1i}(\Omega) Y\big)^{\T},\quad i \in I,
\intertext{
where $d f_{1i}(\Omega) Y$ is given by
}
(d R_{f_{1i}(\Omega)} Y) S_{i}
&= \big((d R_{f_{1}(\Omega)} Y) S\big)_{i},\quad i\in I
\end{align}
\end{subequations}
for any $S = (\dotsc,S_{i},\dotsc)^{\T}\in\R^{|I|\times |J|}$ and by \eqref{eq:df-A1-a}.
\end{lemma}
\noindent
\textit{Proof:} see Appendix \ref{app:closed-form-gradient}. \\[0.1cm]

We focus on the differential of the vector-valued mapping $v_{T}(\Omega)\in\R^{n}$
of~\eqref{eq:Omega-dependency} with $n$ given by~\eqref{eq:vvec-v-c}.
We utilize the fact that analogous to~\eqref{eq:expm-0}, the vector
\begin{subequations}\label{eq:vT-Omega-by-mcA-Omega}
\begin{align}
v_{T}(\Omega)
&=T\vphi(T A^{J}(\Omega))b(\Omega)
= (I_{n},0_{n}) \expm\big(\mc{A}(\Omega)\big)e_{n+1}
\intertext{
can be extracted from the last column
of the matrix} \label{eq:expm-AOmega}
\expm\big(\mc{A}(\Omega)\big) &= \bpm
\expm\big(T A^{J}(\Omega)\big) & v_{T}(\Omega) \\
0_{n}^{\T} & 1
\epm,\qquad
\mc{A}(\Omega) = \bpm T A^{J}(\Omega) & T b(\Omega) \\
0_{n}^{\T} & 0 \epm.
\end{align}
\end{subequations}
By means of relation~\eqref{eq:vecr-F-f-vecr}, we associate a vector-valued
function $f_{\mc{A}}$ with the matrix-valued mapping $\mc{A}$ through
\begin{equation}\label{eq:def-fA}
    \vvec_{r}\big(\mc{A}(\Omega)\big)
    = f_{\mc{A}}\big(\vvec_{r}(\Omega)\big)
\end{equation}
and record for later that, for any matrix $Y\in\R^{|I|\times|I|}$, equation~\eqref{eq:action-dF-df} implies
\begin{equation}\label{eq:dA-Omega}
    \vvec_{r}\big(d\mc{A}(\Omega)Y\big)
    = df_{\mc{A}}\big(\vvec_{r}(\Omega)\big)\vvec_{r}(Y).
\end{equation}
\begin{lemma}\label{lem:dmcA-Omega}
The differential of the mapping $\mc{A}$ in \eqref{eq:expm-AOmega} is given by
\begin{equation}\label{eq:dmcA-Omega}
d\mc{A}(\Omega) Y = T \bpm
d f_{3}(\Omega) & d f_{2}(\Omega) \\
0_{n}^{\T} & 0
\epm \left(\bpm 1 \\ 1 \epm \otimes Y\right),\qquad \forall Y\in\R^{|I|\times |I|}.
\end{equation}
\end{lemma}
\begin{proof}
Equation \eqref{eq:dmcA-Omega} is immediate due to
\begin{equation}
d\mc{A}(\Omega) = \bpm
T d A^{J}(\Omega) Y & T d b(\Omega) Y \\
0_{n}^{\T} & 0
\epm
\end{equation}
and Lemmata \ref{lem:A2} and \ref{lem:A3}.
\end{proof}
Now we are in the position to specify the differential of the solution to the linearized assignment flow with respect to the regularizing weight parameters.
\begin{proposition}\label{prop:diff-vT}
Let
\begin{equation}\label{eq:def-f-vT}
f_{4}(\Omega) := v_{T}(\Omega):=v(T;\Omega)
\end{equation}
denote the solution \eqref{eq:vvec-v-c} in
vectorized form to the ODE~\eqref{eq:LAF-V-b}.
Then the differential is
given according to the convention~\eqref{eq:dF-by-df} by
\begin{subequations}\label{eq:dvT-Omega-explizit}
\begin{align}\label{eq:dvT-Omega-explizit-a}
&d f_{4}(\Omega)Y
= T \Big(
d\big(\vphi\big(T A^{J}(\Omega)\big)b(\Omega)\big)
+ \vphi\big(T A^{J}(\Omega)\big) d f_{2}(\Omega)\Big) Y
\intertext{where}
&d\big(\vphi\big(T A^{J}(\Omega)\big)b(\Omega)\big) Y \label{eq:dvT-Omega-explizit-b} \\
\label{eq:dvT-Omega-explizit-c}
&= \Big(\big(\expm(T A^{J}(\Omega)),v_{T}(\Omega)\big)\otimes e_{n+1}^{\T}\Big)\vphi\big(-\mc{A}(\Omega)\oplus \mc{A}(\Omega)^{\T}\big) \cdot
df_{\mc{A}}\big(\vvec_{r}(\Omega)\big)\vvec_{r}(Y),
\\ \label{eq:dvT-Omega-explizit-d}
&\qquad
\forall Y\in\R^{|I|\times |I|},
\end{align}
\end{subequations}
where $A^{J}(\Omega)$ is given by~\eqref{eq:vvec-v-b}, $\mc{A}(\Omega)$
    by~\eqref{eq:expm-AOmega}, $d f_{\mc{A}}$ by \eqref{eq:dA-Omega} and Lemma \ref{lem:dmcA-Omega}, and $df_{2}$ by Lemma \ref{lem:A2}.
\end{proposition}
\begin{proof}
Equation \eqref{eq:dvT-Omega-explizit-a} follows directly from equation \eqref{eq:vvec-v-c} and Lemma \ref{lem:A2} makes explicit the second summand on the right-hand side. It remains to compute the first summand.
Using~\eqref{eq:vT-Omega-by-mcA-Omega} and the chain rule, we have for any $Y\in\R^{|I|\times|I|}$,
\begin{subequations}
\begin{align}
d\big(T\vphi(T A^{J}(\Omega))b(\Omega)\big) Y
&= (I_{n},0_{n}) d\expm\big(\mc{A}(\Omega)\big)\big(d\mc{A}(\Omega)Y \big) e_{n+1}.
\intertext{
Applying $\vvec_{r}$ to both sides which does not change the vector on the left-hand side, yields by~\eqref{eq:def-vecr}}
d\big(T\vphi(T A^{J}(\Omega))b(\Omega)\big) Y           &= \big((I_{n},0_{n})\otimes e_{n+1}^{\T}\big)\vvec_{r}\big(d\expm\big(\mc{A}(\Omega)\big)(d\mc{A}(\Omega) Y)\big).
            \intertext{Applying Lemma~\ref{lem:dexpm} and~\eqref{eq:dA-Omega}, we obtain}
d\big(T\vphi(T A^{J}(\Omega))b(\Omega)\big) Y            &= \big((I_{n},0_{n})\otimes e_{n+1}^{\T}\big)
            \big(\expm\big(\mc{A}(\Omega)\big)\otimes I_{n+1}\big)
            \vphi\big(-\mc{A}(\Omega)\oplus \mc{A}(\Omega)^{\T}\big)
            \\ &\qquad
            \cdot df_{\mc{A}}\big(\vvec_{r}(\Omega)\big)\vvec_{r}(Y)
            \intertext{and using~\eqref{eq:otimes-mixed} and~\eqref{eq:expm-AOmega}}
            &= \Big(\big(\expm(T A^{J}(\Omega)),v_{T}(\Omega)\big)\otimes e_{n+1}^{\T}\Big)\vphi\big(-\mc{A}(\Omega)\oplus \mc{A}(\Omega)^{\T}\big)
            \\ &\qquad
            \cdot df_{\mc{A}}\big(\vvec_{r}(\Omega)\big)\vvec_{r}(Y).
            \qedhere
        \end{align}
    \end{subequations}
\end{proof}
We finally consider the regularizing mapping $\mc{R}(\Omega)$, defined by \eqref{eq:R-Omega} and corresponding to mapping (M5) in diagram \eqref{eq:Omega-dependency}. Here, we have to take into account the constraints \eqref{eq:def-Omega} imposed on $\Omega$. Accordingly, we define the corresponding set of tangent matrices
\begin{equation}\label{eq:def-mcY-Omega}
\mc{Y}_{\Omega} = \big\{Y\in\R^{|I|\times |I|}\colon \la\eins_{\mc{N}}, Y_{i}|_{\mc{N}}\ra = 0,\;\forall i\in I\big\}.
\end{equation}
\begin{lemma}\label{lem:dmcR-Omega}
The differential of the mapping $\mc{R}$ in \eqref{eq:R-Omega} is given by
\begin{equation}\label{eq:dmcR-Omega}
d\mc{R}(\Omega) Y
= \tau\sum_{i\in I}\Big\la t_{i}(\Omega), \Pi_{0}\Big(\frac{Y_{i}}{\Omega_{i}}\Big)\Big|_{\mc{N}}\Big\ra,\qquad\forall Y\in\mc{Y}_{\Omega}.
\end{equation}
\end{lemma}
\noindent
\textit{Proof:} see Appendix \ref{app:closed-form-gradient}. \\[0.1cm]

Putting all results together, we state the main result of this section.
\begin{theorem}[\textbf{loss function gradient}]\label{thm:loss-function-gradient}
Let
\begin{equation}
\mc{L}(\Omega)
= f_{\mc{L}}\big(v_{T}(\Omega)\big)+\mc{R}(\Omega)
\end{equation}
be a continuously differentiable loss function, where $v_{T}(\Omega)$ given by \eqref{eq:vvec-v-c} is the vectorized solution to the linearized assignment flow \eqref{eq:LAF-V-b} at time $t=T$. Then its gradient $\partial \mc{L}(\Omega)$ is given by
\begin{subequations}\label{eq:partial-mcL-explicit}
\begin{align}\label{eq:partial-mcL-explicit-a}
\la \partial\mc{L}(\Omega),Y\ra
&= d\mc{L}(\Omega) Y,\qquad\forall Y\in\mc{Y}_{\Omega}
\intertext{with}\label{eq:partial-mcL-explicit-b}
d\mc{L}(\Omega)Y
&= \big\la\partial f_{\mc{L}}\big(v_{T}(\Omega)\big), d f_{4}(\Omega) Y\big\ra + d\mc{R}(\Omega) Y
\end{align}
\end{subequations}
and $d f_{4}(\Omega)$ given by \eqref{eq:dvT-Omega-explizit}, and with $d\mc{R}(\Omega) Y$ given by Lemma \ref{lem:dmcR-Omega}.
\end{theorem}
\begin{proof}
The claim \eqref{eq:partial-mcL-explicit} follows from applying the definition of the gradient in \eqref{eq:partial-mcL-explicit-a} and evaluating the right-hand side using the chain rule and Proposition \ref{prop:diff-vT}, to obtain \eqref{eq:partial-mcL-explicit-b}.
\end{proof}

\subsection{Gradient Approximation}
\label{sec:Gradient-Approximation}

In this section, we discuss the complexity of the evaluation of the loss function gradient $\partial\mc{L}(\Omega)$ as given by \eqref{eq:partial-mcL-explicit}, and we develop a low-rank approximation \eqref{eq:approx-gradient-final} that is computationally feasible and efficient.

\subsubsection{Motivation}

We reconsider the gradient $\partial\mc{L}$ given by \eqref{eq:partial-mcL-explicit}.
The gradient involves the term $df_{4}(\Omega)Y$, given by \eqref{eq:dvT-Omega-explizit}, which comprises two summands. We focus on the computationally expensive first summand on the right-hand side of \eqref{eq:dvT-Omega-explizit-a} given by \eqref{eq:dvT-Omega-explizit-b}-\eqref{eq:dvT-Omega-explizit-c}, i.e., the term
\begin{alignat}{1}
	\underbrace{
		\Big(\big(\expm(T A^{J}(\Omega)),v_{T}(\Omega)\big)\otimes e_{n+1}^{\T}\Big)\vphi\big(-\mc{A}(\Omega)\oplus \mc{A}(\Omega)^{\T}\big) \cdot
		df_{\mc{A}}\big(\vvec_{r}(\Omega)\big)
	}_{=: C(\Omega)}
	\vvec_{r}(Y).
\end{alignat}
In order to evaluate the corresponding component of $\partial\mc{L}(\Omega)$ based on \eqref{eq:partial-mcL-explicit-b}, the matrix $C(\Omega)$ is transposed and multiplied
with $\partial f_{\mc{L}}(v_{T}(\Omega))$,
\begin{subequations}\label{eq:CT-partial-f}
    \begin{align}
        &C(\Omega)^{\T} \partial f_{\mc{L}}(v_{T}(\Omega))
        \\
        &= df_{\mc{A}}\big(\vvec_{r}(\Omega)\big)^{\T}
        \vphi\big(-\mc{A}(\Omega)^{\T}\oplus \mc{A}(\Omega)\big) \cdot
        \Big(\big(\expm(T A^{J}(\Omega)),v_{T}(\Omega)\big)^{\T}\otimes e_{n+1}\Big)\partial f_{\mc{L}}(v_{T}(\Omega))
        \\
        &= df_{\mc{A}}\big(\vvec_{r}(\Omega)\big)^{\T}
        \vphi\big(-\mc{A}(\Omega)^{\T}\oplus \mc{A}(\Omega)\big) \cdot
        \Big(\big(\expm(T A^{J}(\Omega)),v_{T}(\Omega)\big)^{\T}\otimes e_{n+1}\Big)
        \big(\partial f_{\mc{L}}(v_{T}(\Omega)) \otimes (1)\big)
        \\
        \overset{\eqref{eq:otimes-mixed}}&{=}
        df_{\mc{A}}\big(\vvec_{r}(\Omega)\big)^{\T}
        \vphi\big(-\mc{A}(\Omega)^{\T}\oplus \mc{A}(\Omega)\big) \cdot
        \Big(\big(\expm(T A^{J}(\Omega)),v_{T}(\Omega)\big)^{\T}
        \partial f_{\mc{L}}(v_{T}(\Omega)) \otimes e_{n+1}\Big).
    \end{align}
\end{subequations}
Thus, the matrix-valued function $\vphi$ defined by \eqref{eq:V-t-vphi} has to be evaluates at a Kronecker sum of matrices and then multiplied by a vector.
The structure of this expression has the general form
\begin{align}\label{eq:BS_Expression}
    f(M_{1}\oplus & M_{2}) (b_1 \otimes b_2),\qquad
    M_1, M_2 \in \Reals^{k \times k},\quad
    b_1, b_2 \in \Reals^k,
\end{align}
where in our case we have
\begin{subequations}\label{eq:BS_Expression_Notation}
\begin{align}
M_1 &= -{\mc{A}(\Omega)}^{\T},\qquad
M_2 = \mc{A}(\Omega), \qquad
k = n+1= |I||J| + 1,
\label{eq:def-k-gradient-component} \\ \label{eq:Benzi-notation-b1}
b_1 &= \big(\expm(T A^{J}(\Omega)),v_{T}(\Omega)\big)^\top \partial f_{\mc{L}}\big(v_{T}(\Omega)\big), \qquad
b_2 = e_{n+1}, \\
f &= \vphi.
\end{align}
\end{subequations}
As the following discussions also hold in the general setting~\eqref{eq:BS_Expression}, we derive our gradient approximation in this full generality. Afterwards, we apply our setting to the gradient approximation \eqref{eq:approx-gradient-final}.
First, we discuss two ways to compute \eqref{eq:BS_Expression}:
\begin{description}
    \item[Direct computation] Compute the Kronecker sum $M_{1} \oplus M_{2}$,
        evaluate the matrix function $\vphi$ and multiply the vector $b_1 \otimes b_2$.
        This approach has space and time complexity of at
        least $\mc{O}(k^4)$, with $k$ given by \eqref{eq:def-k-gradient-component}.
        The complexity might be even higher depending on how the function $f$ is
        evaluated.
    \item[Krylov subspace approximation] Use the Krylov space $\mc{K}_m(M_{1}\oplus M_{2}, b_1 \otimes b_2)$ for
        approximating \eqref{eq:BS_Expression}, as explained in
        Section~\ref{ssec:Exponential_Integration}.
        This approach has space complexity $\mc{O}(k^2 m^2)$ and time
        complexity $\mc{O}(k^2(m+1))$~\cite[p.~132]{Saad2011}.
\end{description}

\begin{remark}[\textbf{space complexity}]\label{rmk:StorageNaiveKrylov}
    Consider an image with $512 \times 512$ pixels ($|I| = 262\,144$),
    $|J|=10$ labels (i.e. $k = |I||J| + 1 = 2\,621\,441$) and using $8$ bytes per
    number. Then the direct computation requires to store more than $10^{14}$ terabytes of data.
    The Krylov subspace approximation (with $m=10$) is significantly cheaper, but still
    requires to store more than $5000$ terabytes.
    Hence both methods are computationally infeasible especially in view of the fact
    that~\eqref{eq:BS_Expression} has to be recomputed in every step of the
    gradient descent procedure \eqref{eq:R-descent-iteration}.
\end{remark}

\subsubsection{An Approximation by Benzi and Simoncini}
To reduce the memory footprint, we employ an approximation for
computing~\eqref{eq:BS_Expression}, first discussed by Benzi and
Simoncini~\cite{Benzi2017}, and refine it using a new additional approximation in
Section~\ref{ssec:Low_Rank_Approximation}.
In the following, the notation from Benzi and Simoncini is slightly adapted to our definition \eqref{eq:def-Kronecker-sum} of the Kronecker sum that differs from Benzi and Simoncini's definition of the Kronecker sum
($A \oplus B = B \otimes I + I \otimes A$).

The approach uses the Arnoldi iteration \cite{Saad2003}
to determine orthonormal bases $P_m$, $Q_m$ and the corresponding Hessenberg
matrices $T_1$ and $T_2$ of the two Krylov subspaces $\mc{K}(M_1, b_1)$, $\mc{K}(M_2, b_2)$.
The matrices are connected by a standard relation of Krylov subspaces~\cite[Section 13.2.1]{Higham2008},
\begin{subequations}\label{eq:Krylov-approx}
\begin{align}
    M_1 P_{m}
    &= P_{m} T_{1} + t_{1} p_{m+1} e_{m}^{\T},
    \\
    M_2 Q_{m}
    &= Q_{m} T_{2} + t_{2} q_{m+1} e_{m}^{\T},
\end{align}
\end{subequations}
where $t_1 \in \R$, $p_{m+1} \in \R^n$ (resp. $t_2 \in \R$, $q_{m+1} \in \R^n$)
refer to the entries of the Hessenberg matrices and the orthonormal bases in the
next step of the Arnoldi iteration.
With these formulas we deduce
\begin{subequations}
\begin{align}
    (M_1 \oplus M_2) (P_m \otimes Q_m)
    \overset{\eqref{eq:def-Kronecker-sum}}&{=}
    (M_1P_m \otimes Q_m) + (P_m \otimes M_2 Q_m) \\
    \overset{\eqref{eq:Krylov-approx}}&{=}
    (P_{m} T_{1} + t_{1} p_{m+1} e_{m}^{\T} \otimes Q_m) + (P_m \otimes Q_{m} T_{2} + P_m \otimes t_{2} q_{m+1} e_{m}^{\T}) \\
    &= (P_m \otimes Q_m)(T_1 \oplus T_2) +
    ( t_{1} p_{m+1} e_{m}^{\T} \otimes Q_m) + (P_m \otimes t_{2} q_{m+1} e_{m}^{\T}).
\end{align}
\end{subequations}
Ignoring the last two summands and multiplying by $(P_m \otimes Q_m)^\top$ yields
the approximation
\begin{alignat}{1}
    (M_1 \oplus M_2) &\approx (P_m \otimes Q_m)(T_1 \oplus T_2)(P_m \otimes Q_m)^\top,
\end{alignat}
which after applying $f$ and multiplying $b_1 \otimes b_2$ leads to the approximation
\begin{alignat}{1}
    f(M_{1}\oplus M_{2}) (b_1 \otimes b_2) \approx (P_m \otimes Q_m)f(T_1 \oplus T_2)(P_m \otimes Q_m)^\T (b_1 \otimes b_2)
\end{alignat}
of the expression \eqref{eq:BS_Expression} as proposed
by Benzi and Simoncini.
We note that, due to the orthonormality of the bases $P_m$ and $Q_m$ and their relation to the vectors $b_{1}, b_{2}$ that generate the subspaces $\mc{K}(M_1, b_1)$, $\mc{K}(M_2, b_2)$, the approximation simplifies to
\begin{subequations}
\begin{align}
    f(M_{1}\oplus M_{2}) (b_1 \otimes b_2) &\approx \|b_1\|\|b_2\| (P_m \otimes Q_m)f(T_1 \oplus T_2) e_1\\
    &= \|b_1\|\|b_2\| \vvec_r\left(P_m \ \vvec_r^{-1}\big( f(T_1 \oplus T_2) e_1 \big) Q_m^\top\right),\label{eq:BSApproximation}
\end{align}
\end{subequations}
where $e_1 \in \R^{m^2}$ denotes the first unit vector.

\begin{remark}[\textbf{complexity of the approximation \eqref{eq:BSApproximation}}]
    Computing and storing the matrices $P_m$, $Q_m$, $T_1$ and $T_2$
    has space complexity $\mc{O}(2 k m^2)$ and a time
    complexity of $\mc{O}(2k(m+1))$~\cite[p.~132]{Saad2011}.
    Storing the matrices $T_1 \oplus T_2$ and $f(T_1 \oplus T_2)$ has complexity
    $\mc{O}(m^4)$.
    Finally, multiplying the three matrices $P_m \in \Reals^{k \times m}$, $\vvec_r^{-1}\left( f(T_1 \oplus T_2) e_1 \right) \in \Reals^{m \times m}$
    and $Q_m^\top \in \Reals^{m \times k}$ has time complexity
    $\mc{O}(k^2m + km^2)$ and space complexity $\mc{O}(k^2 + km)$.

    Ignoring negligible terms (recall $m \ll k$), the entire approximation has computational complexity $\mc{O}(k^2m)$ and storage
    complexity $\mc{O}(k^2)$.
    Compared to the Krylov subspace approximation of \eqref{eq:BS_Expression} discussed in the preceding section, 
this is a reduction of space complexity by a factor $m^2$.

    Consider an image with
    $512 \times 512$ pixels ($|I| = 262\,144$) and $|J|=10$ labels as in Remark~\ref{rmk:StorageNaiveKrylov}. Then the
    approximation~\eqref{eq:BSApproximation} requires to store a bit more than $50$ terabytes.
    While this is a huge improvement compared to the $5000$ terabytes from the
    Krylov approximation (see Remark~\ref{rmk:StorageNaiveKrylov}), using this
    approximation is still computationally infeasible.
    This motivates why we introduce below an additional low-rank approximation that
    yields a computationally feasible and efficient gradient approximation.
\end{remark}

\subsubsection{Low-Rank Approximation}\label{ssec:Low_Rank_Approximation}
We consider again the approximation~\eqref{eq:BSApproximation}
\begin{alignat}{1}
    f(M_{1}\oplus M_{2}) (b_1 \otimes b_2) &\approx \|b_1\|\|b_2\| \vvec_r\left(P_m \ \vvec_r^{-1}\big( f(T_1 \oplus T_2) e_1 \big) Q_m^\top\right)
\end{alignat}
and decompose the matrix
$\vvec_r^{-1}\big( f(T_1 \oplus T_2) e_1 \big) \in \Reals^{m \times m}$
using the singular value decomposition (SVD)
\begin{equation}
    \vvec_r^{-1}\big( f(T_1 \oplus T_2) e_1 \big)
    = \sum_{i \in [m] }\sigma_{i} y_{i}\otimes z_{i}^{\T},
\end{equation}
with $y_{i}, z_{i} \in \Reals^{m}$ and the singular values $\sigma_{i} \in \Reals,\; i\in[m]$.
As $m$ is generally quite small, computing the SVD is
neither computationally nor storage-wise expensive.
We accordingly rewrite the approximation in the form
\begin{subequations}
\begin{align}
    \|b_1\|\|b_2\| &\vvec_r\Big(P_m \ \vvec_r^{-1}\big( f(T_1 \oplus T_2) e_1 \big) Q_m^\top\Big) \\
    =~&\|b_1\|\|b_2\| \vvec_r\bigg(P_m \Big( \sum_{i \in [m] }\sigma_{i} y_{i}\otimes z_{i}^{\T} \Big) Q_m^\top\bigg) \\
    =~&\|b_1\|\|b_2\| \sum_{i \in [m] } \sigma_{i} (P_m y_{i}) \otimes (Q_m z_{i}). \label{eq:factorizedForm}
\end{align}
\end{subequations}
\begin{remark}[\textbf{space complexity}]
    While the factorized form~\eqref{eq:factorizedForm} is equal to the
    approximation~\eqref{eq:BSApproximation}, it requires only a fraction of the
    storage space:
    The intermediate results require storing $m$ singular values and $k$ numbers
    for each $P_m y_{i}$ and $Q_m z_{i}$, and the final approximation has an
    additional storage requirement of $\mathcal{O}(2 k m)$.
    In total $\mathcal{O}(4 k m)$ numbers need to be stored.

    For a $512 \times 512$ pixels image with $10$ labels (see
    Remark~\ref{rmk:StorageNaiveKrylov}), storing this
    approximation requires at most a gigabyte of memory.
\end{remark}

In practice, this can be further improved: Numerical experiments show that the singular values decay very rapidly, such that just the first singular value can be used to obtain the gradient approximation
\begin{alignat}{1}\label{eq:approx-Benzi-plus}
    f(M_{1}\oplus M_{2}) (b_1 \otimes b_2)
    \approx \|b_1\|\|b_2\| \sigma_{1} (P_m y_{1}) \otimes (Q_m z_{1}).
\end{alignat}
Numerical results in Section \ref{sec:Experiments} demonstrate that this approximation is sufficiently accurate.
\begin{remark}[\textbf{space complexity}]
    The term $\|b_1\|\|b_2\| \sigma_{1} (P_m y_{1}) \otimes (Q_m z_{1})$ requires
    to store $\mathcal{O}(2k)$ numbers, i.e.~about twice as much storage space
    as the original image.
    In total, we need to store $\mathcal{O}(2k + 2 k m)$ numbers.
    The required storage for the running example (see
    Remark~\ref{rmk:StorageNaiveKrylov}) now adds up to less than 500 megabytes.
\end{remark}
We conclude this section by returning to our problem using the notation \eqref{eq:BS_Expression_Notation} and state the proposed \textit{low-rank approximation of the loss function gradient}. By \eqref{eq:partial-mcL-explicit}, \eqref{eq:CT-partial-f}, \eqref{eq:BS_Expression_Notation} and \eqref{eq:approx-Benzi-plus}, we have
\begin{subequations}\label{eq:approx-gradient-final}
\begin{align}
\partial\mc{L}(\Omega)
&\approx c(\Omega) \cdot
\vvec_{r}^{-1}\Big(
df_{\mc{A}}\big(\vvec_{r}(\Omega)\big)^{\T}
\big(\sigma_{1} (P_m y_{1}) \otimes (Q_m z_{1})\big)
\intertext{where}
c(\Omega) &= \big\|\big(\expm(T A^{J}(\Omega)),v_{T}(\Omega)\big)^\top \partial f_{\mc{L}}\big(v_{T}(\Omega)\big)\big\|,
\\
v_{T}(\Omega) &= v(T;\Omega)\qquad\text{(cf.~\eqref{eq:vvec-v-c})}
\\
\sigma_{1}y_{1}\otimes z_{1}^{\T}
&\approx \vvec_{r}^{-1}\big(\vphi(T_{1}\otimes T_{2})e_{1}\big).
\qquad\text{(largest singular value and vectors)}
\end{align}
\end{subequations}
Here, the matrices $P_{m}, Q_{m}, T_{1}, T_{2}$ result from the Arnoldi iteration, cf.~\eqref{eq:Krylov-approx}, that returns the two Krylov subspaces used to approximate the matrix vector product $\vphi(-\mc{A}(\Omega)^{\T}\oplus \mc{A}(\Omega)) b_{1}$, with $b_{1}$ given by \eqref{eq:Benzi-notation-b1}.

\subsection{Computing the Gradient using Automatic Differentiation}\label{ssec:AutoDiff}
An entirely different approach to computing the gradient $\partial \mc{L}(\Omega)$
of the loss function~\eqref{eq:Omega-loss} is to not use an approximation of
the exact gradient given in closed form by \eqref{thm:loss-function-gradient}, but to replace the solution $v_{T}(\Omega)$ to the linearized assignment flow in \eqref{eq:partial-mcL-explicit-b} by an approximation determined by a numerical
integration scheme and to compute the exact gradient therefrom. Thus, one replaces
a \textit{differentiate-then-approximate} approach by an \textit{approximate-then-differentiate} alternative.
We numerically compare these two approaches in
Section~\ref{sec:Experiments}.

We sketch the latter alternative. Consider again the loss function~\eqref{eq:Omega-loss} evaluated at the
linearized assignment flow integrated up to time $T$
\begin{equation}
    \mc{L}(\Omega) = f_{\mc{L}}\big(v_{T}(\Omega)\big).
\end{equation}
Gradient approximations determined by automatic differentiation depend on what numerical scheme is used. We pick out two basic choices out of a broad range of proper schemes studied in \cite{Zeilmann2020}. In both cases, we implemented the loss function $f_{\mc{L}}$ in PyTorch together with the functions $\Omega \mapsto A^{J}(\Omega)$
and $\Omega \mapsto b(\Omega)$ given by \eqref{eq:vvec-v}. Now two approximations can be distinguished depending on how the mappings $(A^{J}(\Omega), b(\Omega)) \mapsto v_{T}(\Omega)=v(T;\Omega)$ are implemented.
\begin{description}
    \item[Automatic Differentiation based on the explicit Euler scheme]
        We partition the interval $[0,T]$ into $T/h$ subintervals with some
        step size $h>0$ and use the iterative scheme
        \begin{equation}\label{eq:Euler_Integration}
            v^{(k+1)} = v^{(k+1)} + h \big( A^J(\Omega) v^{(k)} + b(\Omega) \big),
            \qquad
            v^{(0)} = 0,
        \end{equation}
        in order to approximate $v_{T}(\Omega) \approx v^{(T/h)}$ the solution to the linearized assignment flow ODE \eqref{eq:vvec-v-a}.
        As the computations only involve basic linear algebra, PyTorch is able
        to compute the gradient using automatic differentiation.
    \item[Automatic Differentiation based on exponential integration]
        The second approximation utilizes the numerical integration scheme developed in Section~\ref{ssec:Exponential_Integration}. Again, only basic operations of linear algebra are involved so that PyTorch can compute the
        gradient using automatic differentiation.
The more special matrix exponential \eqref{eq:expm-0} is computed by PyTorch using a Taylor polynomial approximation~\cite{Bader2019}.
\end{description}

Both approaches determine an approximation of the Euclidean gradient $\partial\mc{L}(\Omega)$ which we subsequently convert
into an approximation of the Riemannian gradient using Equation \eqref{eq:R-grad-mcL}.

\section{Experiments}
\label{sec:Experiments}

In this section, we report and discuss a series of experiments illustrating our
novel gradient approximation~\eqref{eq:approx-gradient-final} and the
applicability of the linearized assignment flow to the image labeling problem.

We start with a discussion of the data generation
(Section~\ref{ssec:Data_Generation}) and the general experimental setup
(Section~\ref{ssec:Experimental_Setup}), before discussing properties of the
gradient approximation (Section~\ref{ssec:PropGradApprox}).
In order to illustrate a complete pipeline that can also label previously unseen
images, we trained a simple parameter predictor and report its application in
Section~\ref{ssec:Parameter_Prediction}.

\subsection{Data Generation}\label{ssec:Data_Generation}

As for the experiments, we focused on the image labeling
scenarios depicted in Figures \ref{fig:Experiments_VoronoiOutline} and \ref{fig:Experiments_ColorVoronoi}. Each scenario consists of a set containing five $128 \times 128$ pixel
images with \textit{random} Voronoi structure, in order to mimic low-dimensional structure that has to be separated in noisy data from the background. This task occurs frequently in applications and cannot be solved without \textit{adaptive} regularization.

For the design of the parameter predictor (Section~\ref{ssec:Parameter_Prediction}), we used all patches of five
additional unseen images for validation.
In all cases we report the mean over all labeled pixels of 5 training and validation images, respectively.
In order to test the resilience to noise, we added Gaussian noise to the images. The ground truth labeling is, in both labeling scenarios, given by the noiseless version of the images.

In the first scenario illustrated by Figure \ref{fig:Experiments_VoronoiOutline}, we want
to separate the boundary of the cells (black label) from their interior (white label).
The main difficulty here is to preserve the thin line structures even in the presence of image noise.
Weight patches with uniform (uninformed) weights average out most of the lines as
Figure~\ref{fig:Experiments_Labeling_uniform} shows.

In the second scenario illustrated by  Figure~\ref{fig:Experiments_ColorVoronoi}, we label the Voronoi
cells according to their color represented by 8 labels.
Due to superimposed noise, a pixelwise local rounding to the nearest label yields about 50\% wrongly labeled pixels, see Figure \ref{fig:Experiments_Labeling_pixelwise}.

\begin{figure}[!ht]
	\centering
	\begin{subfigure}[b]{0.3\textwidth}
		\includegraphics[width=\textwidth]{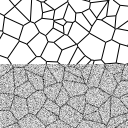}
		\caption{Random Voronoi line structure to be labeled from noisy input data.}
		\label{fig:Experiments_VoronoiOutline}
	\end{subfigure}
	\qquad
	\qquad
	\begin{subfigure}[b]{0.3\textwidth}
		\includegraphics[width=\textwidth]{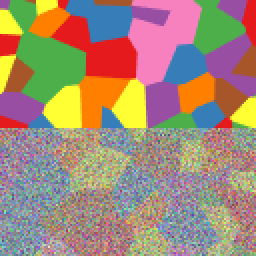}
		\caption{Random colored Voronoi regions to be labeled from noisy input data.}
		\label{fig:Experiments_ColorVoronoi}
	\end{subfigure}\\[5mm]
	\begin{subfigure}[b]{0.3\textwidth}
		\includegraphics[width=\textwidth]{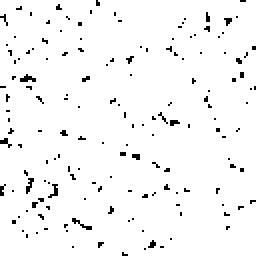}
		\caption{Labeling with uniform weights, that is without weight adaption, cannot separate line structure from the background in noisy data.}
		\label{fig:Experiments_Labeling_uniform}
	\end{subfigure}
	\qquad
	\qquad
	\begin{subfigure}[b]{0.3\textwidth}
		\includegraphics[width=\textwidth]{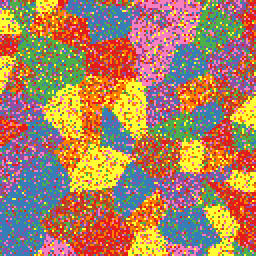}
		\caption{Pixelwise individual nearest label assignments produce an error rate larger than 50\%.}
		\label{fig:Experiments_Labeling_pixelwise}
	\end{subfigure}
	\caption{
		\textbf{Randomized scenarios for training and testing.}
		Two randomly generated images for the two respective scenarios that
		were used to evaluate weight parameter estimation and prediction.
		\textbf{(a)}
		Random line structure whose accurate labeling requires to adapt weight
		parameters.
		\textbf{(b)}
		Random Voronoi cells to be labeled by pixelwise assignment of one of the colors
		(\crule[c1]{1.75mm}{1.75mm}, \crule[c2]{1.75mm}{1.75mm},
		\crule[c3]{1.75mm}{1.75mm}, \crule[c4]{1.75mm}{1.75mm},
		\crule[c5]{1.75mm}{1.75mm}, \crule[c6]{1.75mm}{1.75mm},
		\crule[c7]{1.75mm}{1.75mm}, \crule[c8]{1.75mm}{1.75mm}).\\
		In both cases, Gaussian noise was added.
		The resulting noisy images are shown in the lower part of either panel
		(rescaled in the color channels to avoid color clipping).
		\textbf{(c)}
		The amount of noise is chosen quite large such that a labeling with
		uniform (``uninformed non-adaptive'') weights completely destroys the thin line structure in (a).
		\textbf{(d)}
		A pixelwise local nearest label assignment yields around 50\% wrongly labeled pixels for the
		labeling scenario depicted in (b).
		Both of these naive parameter settings indicate the need for a more
		structured choice of the weight patches, by taking into account local image features in a local spatial neighborhood.
	}\label{fig:Experiments}
\end{figure}

\subsection{Experimental Setup}\label{ssec:Experimental_Setup}

\textbf{Features and Parametrization.} For simplicity, we used the raw image data in a $3 \times 3$ window around each
pixel as feature \eqref{eq:mcF_I} for this pixel.
Weight patches $(\w_{ik})_{k\in\mc{N}_{i}}$ in the $\Omega$-matrix \eqref{eq:def-Omega} also had the size of $3 \times 3$ pixels in all experiments.
While the linearized assignment flow works with arbitrary features and also with larger
neighborhood sizes for the weight parameters,
the above setup suffices to illustrate and substantiate the contribution of this paper.

\textbf{Performance measure.} All labelings were evaluated on the tangent space of the assignment manifold using the loss function $f_\mathcal{L}$ given by \eqref{eq:cosineSimilarity}.
Since the values of this function are rather abstract, however, we report the
percentage of wrongly labeled pixels in all performance plots.

\textbf{Gradient computation.} We evaluated the loss function and approximated its Riemannian gradient
in three different ways, as further detailed in Section~\ref{ssec:PropGradApprox}, throughout using uniform (uninformed) weight patches as initialization. In particular, other common ways to update the parameters, like Adam or AdaMax \cite{Kingma2017}, are
possible as well, in conjunction with our approach. Therefore, we also compared gradient approximations based on our approach with the results of automatic differentiation, as implemented by PyTorch~\cite{Paszke2019}.

\textbf{Parameter prediction.} Parameter prediction for labeling novel data relies on the relation of features extracted from training data to corresponding parameters estimated by the Riemannian gradient descent \eqref{eq:R-descent-iteration}. For any feature extracted from novel data, the predictor specifies the parameters, to be used for labeling the data by integrating the linearized assignment flow after substituting the predicted parameters. Details are provided in Section \ref{ssec:Parameter_Prediction}.

\subsection{Properties of the Gradient Approximation}\label{ssec:PropGradApprox}

In this section, we report results that empirically validate our novel gradient
approximation~\eqref{eq:approx-gradient-final} by means of parameter estimation for the linearized assignment flow.

First, we compared our gradient approximation with two methods based on
automatic differentiation (backpropagation), see also Section~\ref{ssec:AutoDiff}.
To this end, we implemented in
PyTorch~\cite{Paszke2019} the simple explicit Euler
scheme \eqref{eq:Euler_Integration} for integrating the linearized assignment flow and computed the gradient of the loss function
$\mc{L}(\Omega)$ \eqref{eq:Omega-loss} with respect to $\Omega$ using automatic
differentiation.
Similarly, the Krylov subspace approximation~\eqref{eq:Krylov-vphi} of the
solution of the linearized assignment flow was implemented in PyTorch.
As all involved computations in this approximation are basic linear algebra
operations, PyTorch is able to apply automatic differentiation for evaluating
the gradient.

These gradients are used for carrying out the gradient descent iteration \eqref{eq:R-descent-iteration} in order to optimize the weight parameters.
Figure~\ref{fig:Algorithm_Comparison} illustrates the comparison of the three approaches.
Although they rely on quite different principles, we observe a remarkable comparability of the three approaches with respect to
the reduction of the percentage of wrongly labeled pixels per training iteration,
for both noisy and noiseless images. In particular, our low-rank approximation based on the closed-form loss function gradient expression is competitive. In view of the minor differences between the curves, we point out that
changing hyperparameters, like the step size in the gradient descent or the scale parameter $\tau$ of the regularizer $\mc{R}$ in \eqref{eq:R-Omega}, have a greater effect on the training performance
than the choice of either of the three approaches.
Overall, these results validate the closed form formulas in Section~\ref{sec:Loss-Function-Gradient} and, in particular, Theorem \ref{thm:loss-function-gradient}, and the subsequent low-rank approximation in Section~\ref{sec:Gradient-Approximation}. We point out, however, that our approach only reveals data-dependent low-dimensional subspaces where the essential parameters of the linearized assignment flow reside.

\begin{figure}[!ht]
	\centering
	\begin{subfigure}[b]{0.45\textwidth}
		\includegraphics[width=\textwidth]{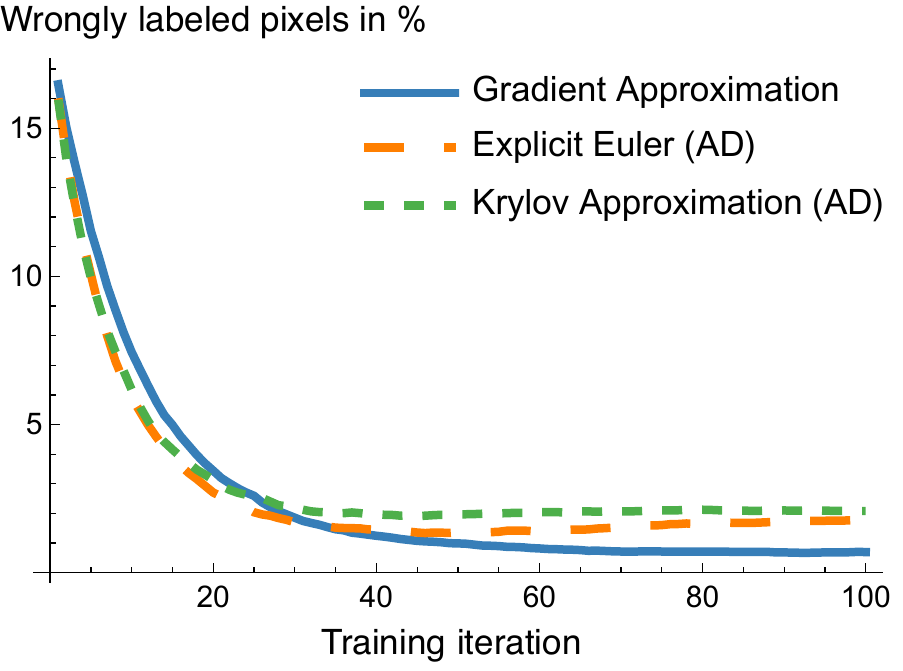}
		\caption{Noisy image}
		\label{fig:Algorithm_Comparison_Noisy_Image}
	\end{subfigure}
	\hfill
	\begin{subfigure}[b]{0.45\textwidth}
		\includegraphics[width=\textwidth]{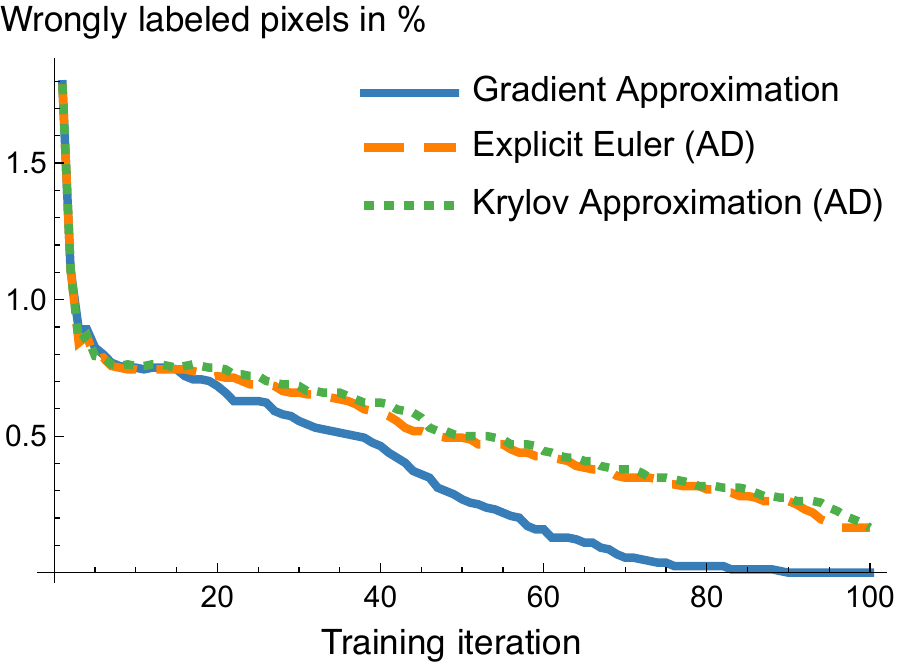}
		\caption{Noiseless image}
		\label{fig:Algorithm_Comparison_Noiseless_Image}
	\end{subfigure}
	\caption{
		\textbf{Comparing gradient approximation and automatic differentiation.}
		Both figures show, for the second scenario depicted by Figure \ref{fig:Experiments_ColorVoronoi}, the effect of parameter learning in terms of the labeling
		error during the training procedure \eqref{eq:R-descent-iteration}. Panel (a) shows the result for noisy input data, panel (b) or noiseless input data.
		Note the different scales of the two ordinates.
		As is exemplarily shown here by both figures, we generally observed very
		similar results for all three algorithms which validates the closed form
		formulas in Section~\ref{sec:Loss-Function-Gradient} and the subsequent
		subspace approximation in Section~\ref{sec:Gradient-Approximation}.
	}\label{fig:Algorithm_Comparison}
\end{figure}

Next, we compared our gradient approximation to the
exact gradient on a per-pixel basis.
However, as the exact gradient is computationally infeasible, we used the
gradient produced by automatic differentiation of the explicit Euler scheme with
a very small step size as surrogate.
Figure~\ref{fig:Gradient_Comparison_cosError} demonstrates the high accuracy
of our gradient approximation.
A pixelwise illustration of the gradient approximation, at the initial step of the training procedure for adapting the parameters,
is provided by Figure~\ref{fig:Gradient_Comparison_gradientNorm}.
The set of pixels with non-zero loss function gradient concentrate around the line structure since here weight adaption is required to achieve a proper labeling.

\begin{figure}[!ht]
	\centering
	\begin{subfigure}[b]{0.4\textwidth}
		\includegraphics[width=0.969\textwidth]{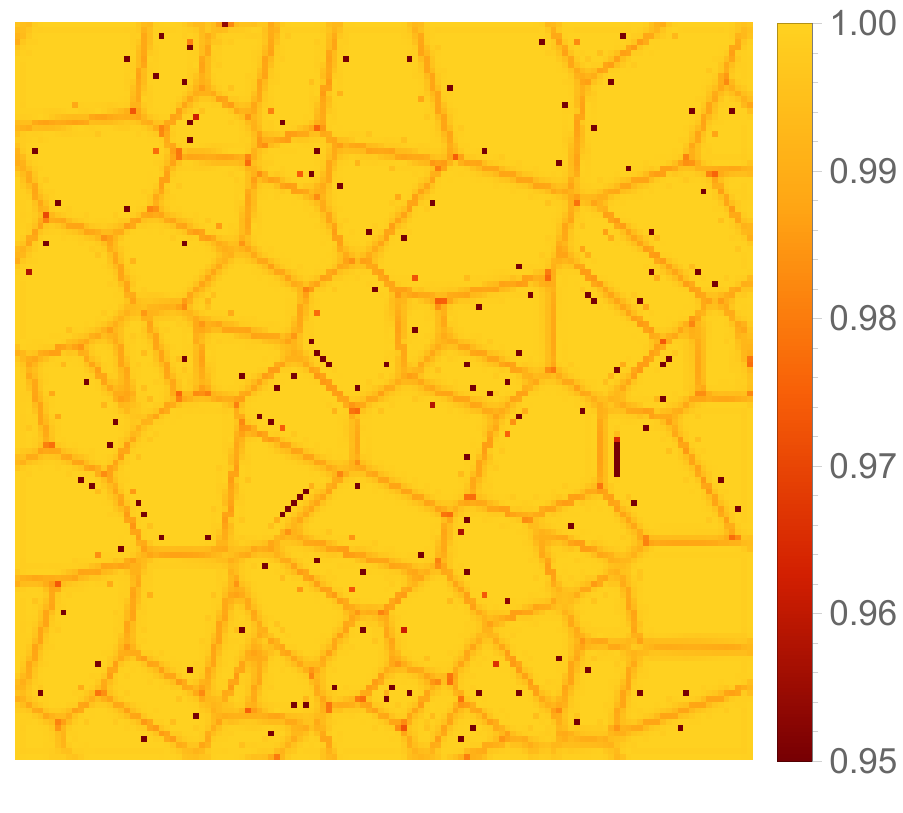}
		\caption{Gradient directions}
		\label{fig:Gradient_Comparison_cosError}
	\end{subfigure}
	\qquad
	\begin{subfigure}[b]{0.4\textwidth}
		\includegraphics[width=\textwidth]{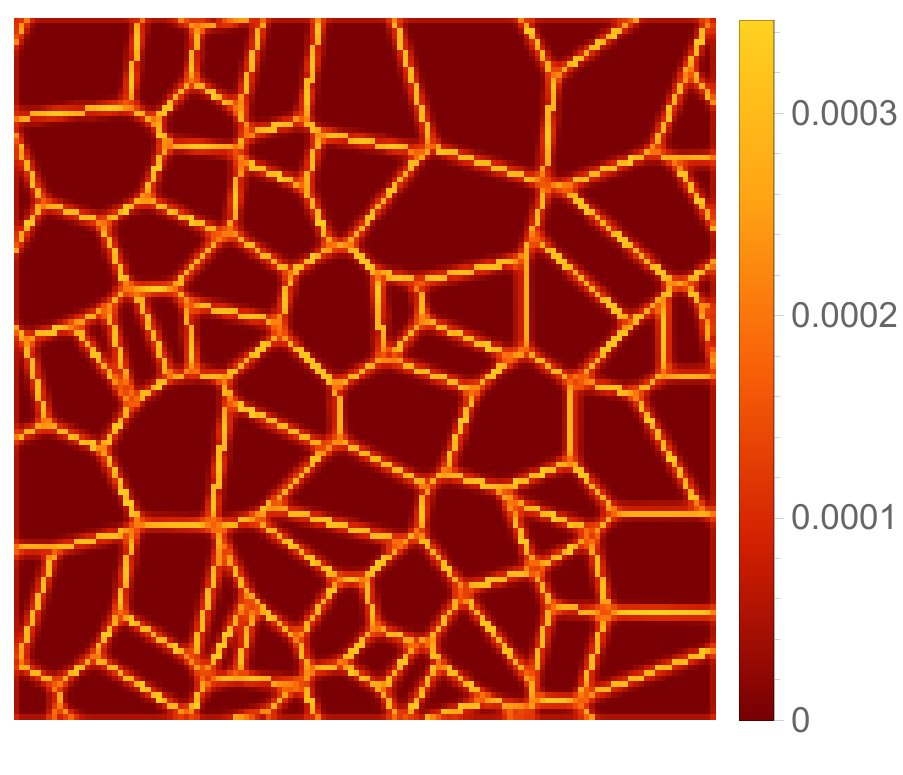}
		\caption{Norm of gradients}
		\label{fig:Gradient_Comparison_gradientNorm}
	\end{subfigure}
	\caption{
		\textbf{Checking the gradient approximation at each pixel.}
		We evaluated our gradient approximation~\eqref{eq:approx-gradient-final},
		at the first step of the training iteration and at each pixel, for the
		scenario depicted in Figure~\ref{fig:Experiments_VoronoiOutline}. As a proxy for the exact but computationally infeasible gradient, we used the gradient produced by automatic differentiation of the
		explicit Euler scheme with a very small step size.
		Then we compared both gradients at each pixel using the cosine
		similarity, i.e.~the value $1$ means that the gradients point exactly in the
		same direction, whereas $0$ signals orthogonality and $-1$ means that they
		point in opposite directions.
		\textbf{(a)} More than 99\% of the pixels have a value of 0.9 or more, corresponding
		to an angle of $26^\circ$ or less between the gradient directions. This
		illustrates excellent agreement between our gradient approximation and the
		exact gradient.
		Disagreements with the exact gradient occur rarely and randomly
		at isolated pixels throughout the image.
		\textbf{(b)} Norm of the gradients are displayed at each pixel.
		Non-vanishing norms indicate where parameter learning (adaption) occurs.
		Since the initial weight parameter patches are uniform, no adaption --
		corresponding to zero norms of gradients -- occurs in the interior of each Voronoi cell, because
		parameters are already optimal in such homogeneous regions.
    }\label{fig:Gradient_Comparison}
\end{figure}

Our last three experiments regarding the
gradient approximation, illustrated by Figure~\ref{fig:ParameterComparison}, concern
\begin{itemize}
\item
the influence of the Krylov dimension $m$,
\item
the rank of our approximation, and
\item
the time $T$ up to which the linearized assignment flow is integrated.
\end{itemize}
We observe according to  Figure \ref{fig:ParameterComparison_Krylov_noisy} that already Krylov subspace of small dimension $m \approx 10$ suffice for computing linearized assignment flows  and learning their parameters.
Similarly, the final rank-one gradient approximation of the gradient according to Eq.~\eqref{eq:approx-Benzi-plus} suffices for parameter estimation, as illustrated in Figure \ref{fig:ParameterComparison_Rank}.
These experiments show that quite \textit{low-dimensional} representations
suffice for representing the information required for optimal regularization of
dynamic image labeling.
We point out that such insights cannot be gained from automatic differentiation.

The influence of the time $T$ used for integrating the linearized assignment flow on parameter learning is illustrated in Figure~\ref{fig:ParameterComparison_Time}.
For the considered parameter estimation setup, we observe
that already small integration times $T$ yield good training results, whereas
large times $T$ yield slower convergence. A possible explanation is that, in the latter case, the linearized assignment flow is close to an integral solution which, when erroneous, is more difficult to correct.

\begin{figure}[!ht]
	\centering
	\begin{subfigure}[b]{0.3\textwidth}
		\includegraphics[width=\textwidth]{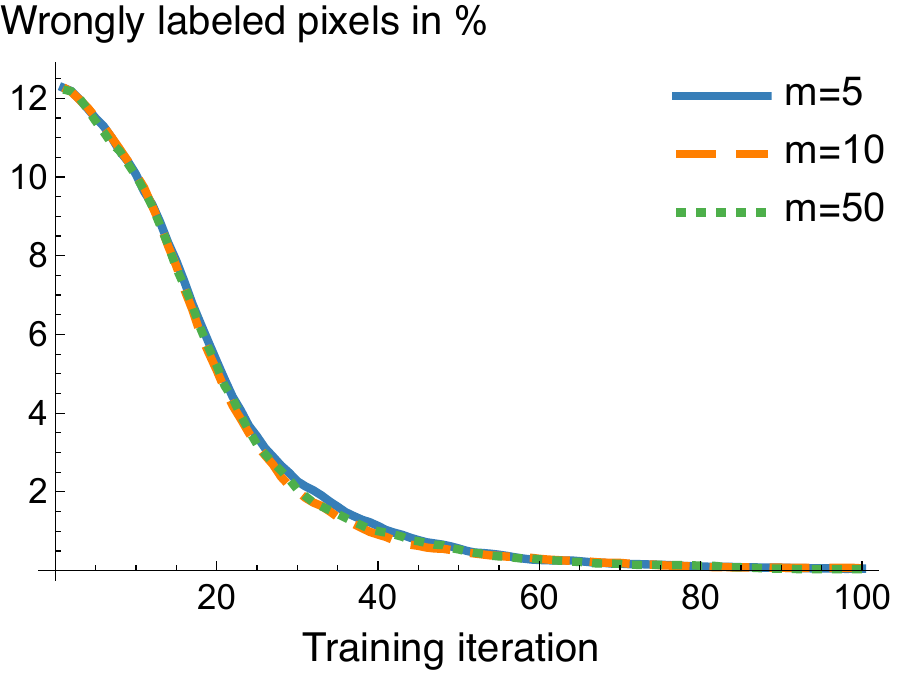}
		\caption{Krylov subspace dimension $m$}
		\label{fig:ParameterComparison_Krylov_noisy}
	\end{subfigure}
	\quad
	\begin{subfigure}[b]{0.3\textwidth}
		\includegraphics[width=\textwidth]{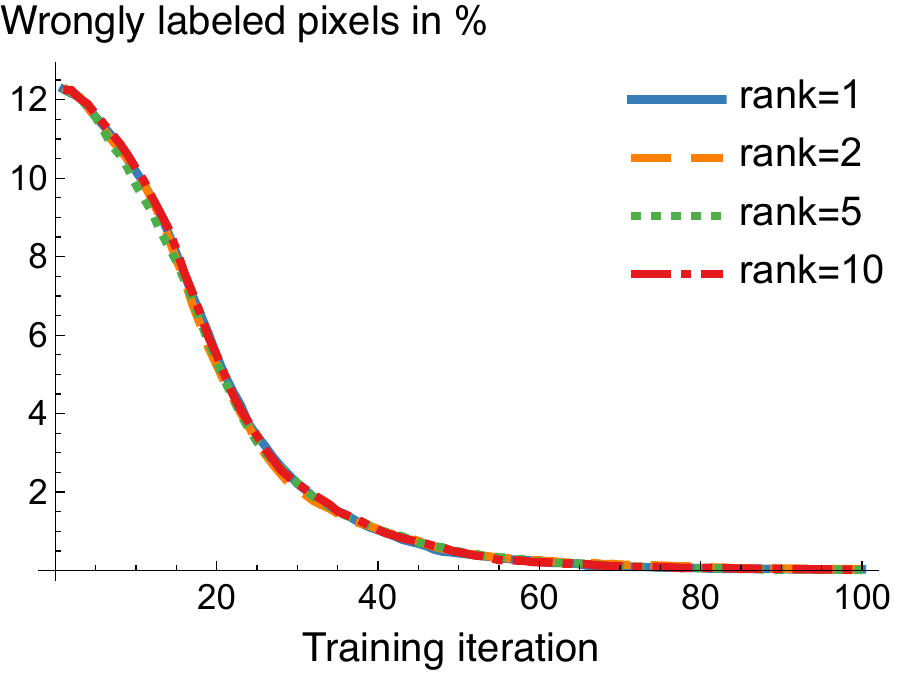}
		\caption{Rank of the approximation}
		\label{fig:ParameterComparison_Rank}
	\end{subfigure}
	\quad
	\begin{subfigure}[b]{0.3\textwidth}
		\includegraphics[width=\textwidth]{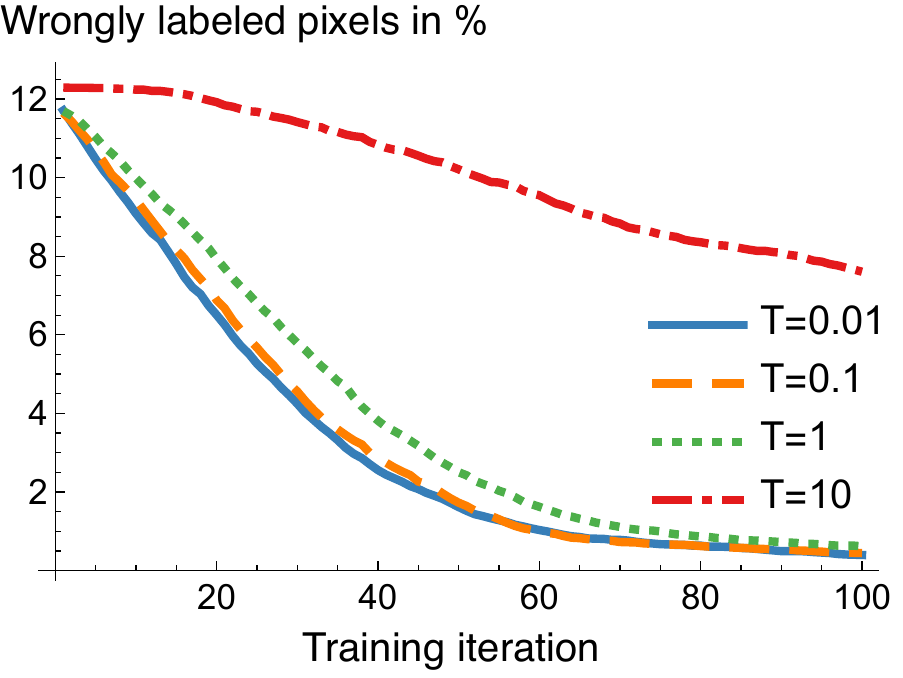}
		\caption{Integration time $T$}
		\label{fig:ParameterComparison_Time}
	\end{subfigure}
	\caption{
		\textbf{Influence of Krylov subspace dimension, rank of the gradient approximation and
		integration time.}
		The setup of Figure~\ref{fig:Algorithm_Comparison} was used to
		demonstrate the influence of the Krylov subspace dimension, the
		low-rank approximation and the integration time $T$ on our gradient
		approximation for parameter learning.
		\textbf{(a)}
		In general, we observed that Krylov dimensions of $5$ to
		$10$ are sufficient for most experiments.
		Larger Krylov dimensions only increase the computation time without
		any noticeable improvement of accuracy.
		\textbf{(b)}
		Training curves for different low-rank approximations coincide.
		This illustrates that just selecting the largest singular value and
		vectors in \eqref{eq:approx-gradient-final}, according to the final rank-one approximation \eqref{eq:approx-Benzi-plus}, suffices for parameter
		learning.
		\textbf{(c)}
		For small integration times $T$, the convergence rates of training do not much differ.
		Only for larger time points $T$, we observe slower convergence of training, presumably because almost hard decisions are more difficult to correct by changing the parameters of the underlying dynamical system.
	}\label{fig:ParameterComparison}
\end{figure}

\subsection{Parameter Prediction}\label{ssec:Parameter_Prediction}
Besides parameter learning, parameter \textit{prediction} for unseen test data
defines another important task.
This task amounts to model and represent the relation of local features and optimal weight parameters, as basis to predict proper weights in unseen test data as a function of corresponding local features. 

We illustrate this for the scenario depicted by
Figure~\ref{fig:Experiments_VoronoiOutline} using the following simple end-to-end
learned approach to parameter prediction.
We trained a predictor that produces a weight patch $\wh{\Omega}_{i}$
given the features $f_i$ at vertex $i$ of novel unseen data. 
The predictor is parameterized with $N=50$ by
\begin{subequations}
\begin{align}
	&p_j \in \mathbb{R}^{3 |\mathcal{N}|},\; j\in [N] 
	&&\text{feature prototypes,}
	\\
	&\nu_j \in T_0,\; j\in [N] 
	&&\text{tangent vectors representing prototypical weight patches,}
\end{align}
\end{subequations}
and a scale parameter $\sigma \in \mathbb{R}$.
Similar to the assignment vectors~\eqref{eq:def-Wi}, the to-be-predicted weight
patches $\wh{\Omega}_{i}$ are elements of the probability simplex $\mathring\Delta_{|\mc{N}_{i}|}$, see~\eqref{eq:def-Omega}.
Accordingly, use tangent vector $\nu_{j}\in T_0$ to
represent weight patches. In particular, tangent vector of \textit{predicted} weight patches result from weight averaging of vectors $\{\nu_{j}\}_{j\in[N]}$, and the predicted weight patch by lifting, see \eqref{eq:weight-prediction}.

We initialize $\sigma = 1$ and initialize the $p_j,\,j\in[N]$ by clustering noise-free patches extracted from of training images.
Given $p_{j}$, we initialize $\nu_j$ such that it is directed towards the label of the corresponding prototypical patch,
\begin{alignat}{1}
    \nu_{j} &= \Pi_0
	\begin{pmatrix}
		e^{ - \| p_{j,1} - p_{j,\text{center pixel}} \| },
		\dots,
		e^{ - \| p_{j,|\mathcal{N}|} - p_{j,\text{center pixel}} \| }
	\end{pmatrix}^\top,\qquad j\in[N].
\end{alignat}

The predictor is trained by the following gradient descent iteration. 
As the change in the number of wrongly labeled pixels was small, we stopped the iteration after $100$ steps, see Figure~\ref{fig:Prediction_Curves}.
\begin{enumerate}[(1)]
	\item We compute the similarities
		\begin{equation}
			s_{ij} = e^{-\sigma \|f_i-p_{j}\|},\qquad j\in[N]
		\end{equation}
		for each $f_i$ and pixels $i$ in all training images.
	\item We predict the corresponding weight patches as lifted weighted average
		of the tangent vectors $\nu_{j}$ 
		\begin{equation}\label{eq:weight-prediction}
			\wh{\Omega}_{i}(\nu, p, \sigma) = \exp_{\eins_{\Omega}}\bigg(\sum_{j\in[N]}\frac{s_{ij}}{\sum_{k\in[N]}s_{ik}} \nu_{j}\bigg).
		\end{equation}
	\item Substituting $\wh{\Omega}$ for $\Omega$, we run the linearized assignment flow and evaluate the distance function~\eqref{eq:cosineSimilarity}.
	\item The gradient of this function with respect to the predictor parameters $(\nu,p,\sigma)$ results from composing the differential due to Theorem \ref{thm:loss-function-gradient} and the differential of \eqref{eq:weight-prediction}. 
	\item The gradient is used to update the predictor parameters, and all steps are repeated.
\end{enumerate}
During training, the accuracy of the predictor is monitored, as illustrated by Figure \ref{fig:Prediction_Curves}. The iteration terminates when the slope of the validation curve, which measures label changes, are sufficiently flat.

After the training of the predictor, the linearized assignment flow is parametrized in a data-driven way so as to separate reliably line structure in noisy data for arbitrary random instances, as depicted by Figure \ref{fig:Prediction}: panel (f) and last row. 
This result should be compared to the non-adaptive labeling result in Figure~\ref{fig:Experiments_Labeling_uniform}.

\begin{figure}
	\centering
	\begin{subfigure}[b]{0.25\textwidth}
		\includegraphics[width=\textwidth]{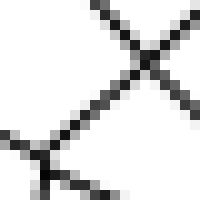}
		\caption{Section of noise-free image}
		\label{fig:Prediction_Noiseless_Input}
	\end{subfigure}\quad
	\begin{subfigure}[b]{0.25\textwidth}
		\includegraphics[width=\textwidth]{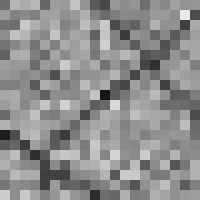}
		\caption{Section of noisy image}
		\label{fig:Prediction_Noisy_Input}
	\end{subfigure}\quad
	\begin{subfigure}[b]{0.4\textwidth}
		\includegraphics[width=\textwidth]{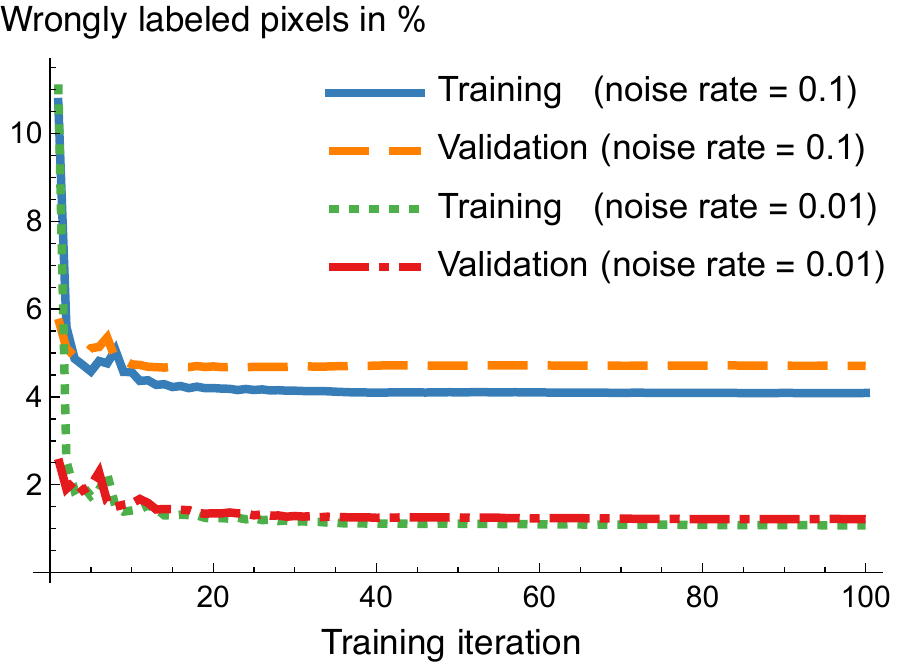}
		\caption{Predictor accuracy}
		\label{fig:Prediction_Curves}
	\end{subfigure}

	\begin{subfigure}[b]{0.3\textwidth}
		\includegraphics[width=\textwidth]{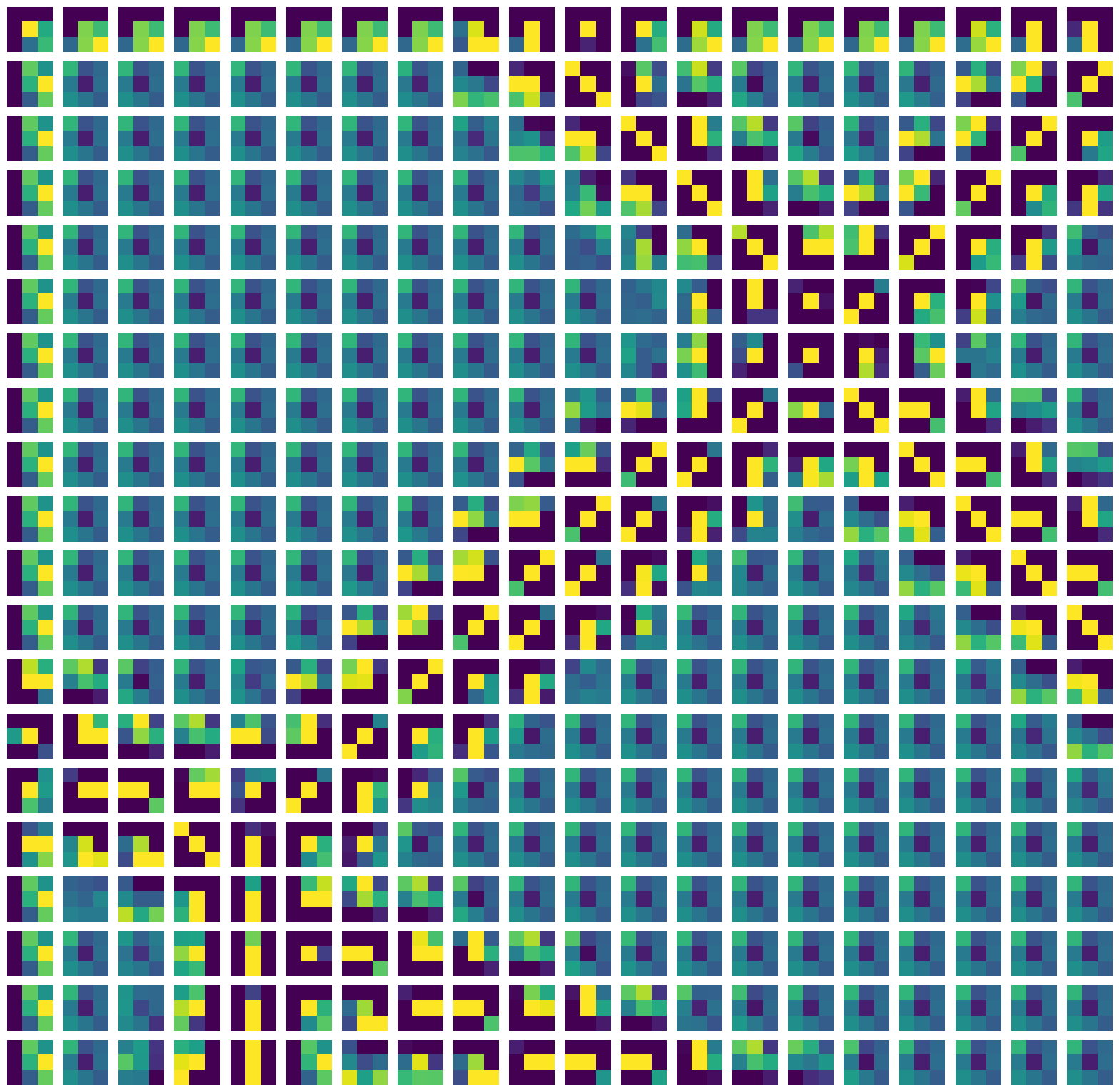}
		\caption{Predicted $\Omega$-weight patches for noiseless input}
		\label{fig:Prediction_Noiseless_Weights}
	\end{subfigure}\quad
	\begin{subfigure}[b]{0.3\textwidth}
		\includegraphics[width=\textwidth]{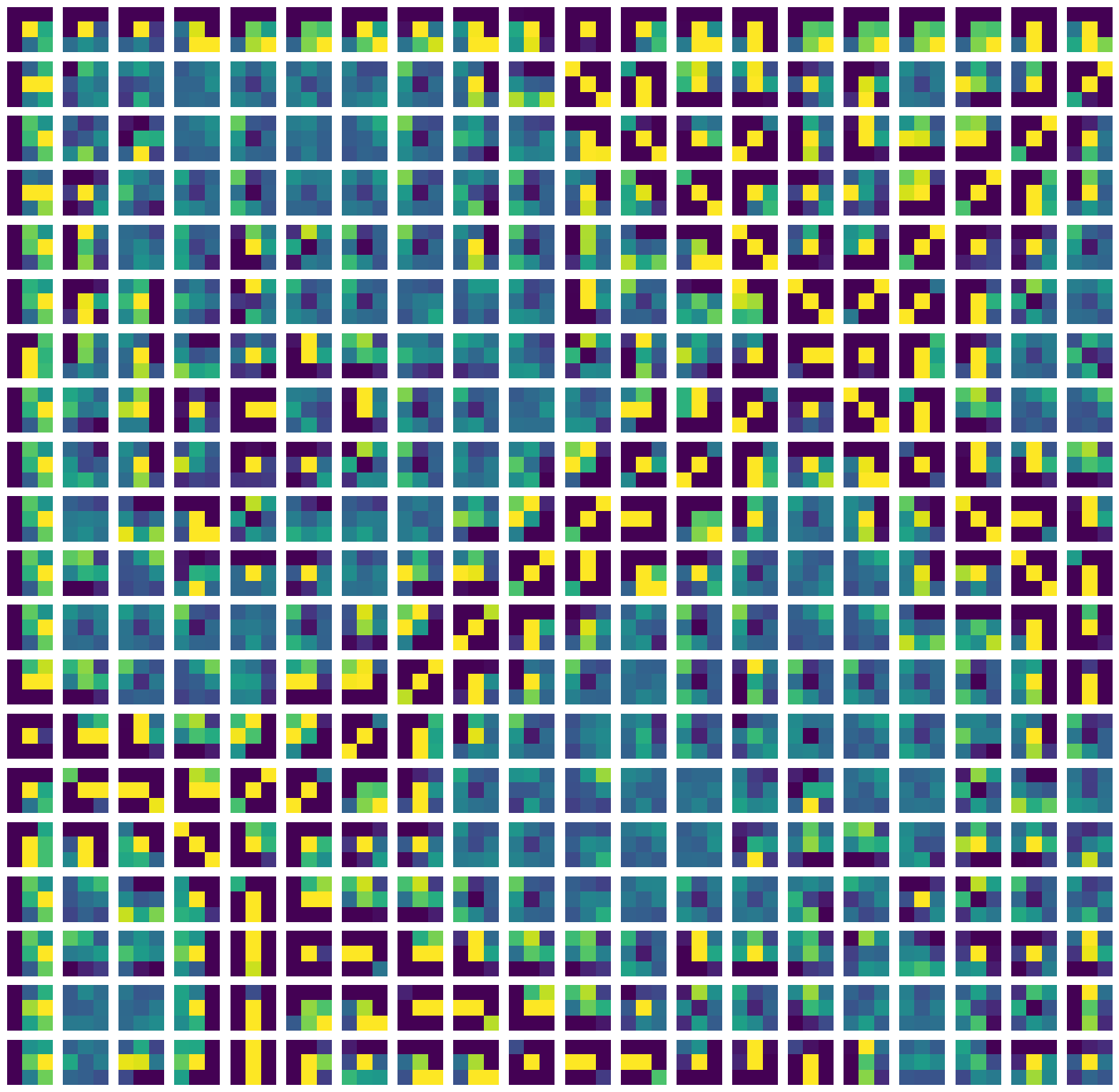}
		\caption{Predicted $\Omega$-weight patches for noisy input}
		\label{fig:Prediction_Noisy_weights}
	\end{subfigure}\quad
	\begin{subfigure}[b]{0.3\textwidth}
		\includegraphics[width=\textwidth]{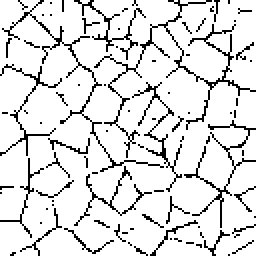}
		\caption{Labeling of the noisy image with predicted weights}
		\label{fig:Prediction_Labeling}
	\end{subfigure}

	\begin{subfigure}[b]{0.3\textwidth}
		\includegraphics[width=\textwidth]{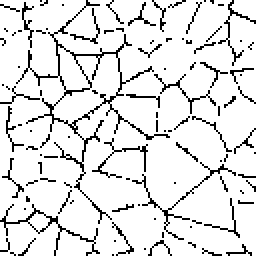}
		\label{fig:Prediction_Labeling_1}
	\end{subfigure}\quad
	\begin{subfigure}[b]{0.3\textwidth}
		\includegraphics[width=\textwidth]{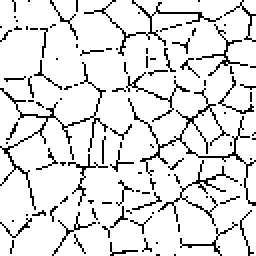}
		\label{fig:Prediction_Labeling_2}
	\end{subfigure}\quad
	\begin{subfigure}[b]{0.3\textwidth}
		\includegraphics[width=\textwidth]{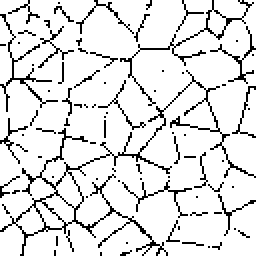}
		\label{fig:Prediction_Labeling_3}
	\end{subfigure}

	\caption{
		\textbf{Parameter Predictor.} We learned a weight patch
		predictor as described in Section \ref{ssec:Parameter_Prediction} for
		the scenario depicted by Figure \ref{fig:Experiments_VoronoiOutline}.
		In order to assess the predicted parameters by comparison, we also
		estimated weights patches for the \textit{noise-free} test data in the
		same way as for the training data.
		\textbf{(a)}
		Section of a noise-free test image.
		\textbf{(b)}
		The corresponding section of the noisy test image that is used as input
		data for prediction.
		\textbf{(c)}
		The training and validation accuracy during the training of the predictor.
		\textbf{(d)}
		Weight patches estimated for the \textit{noise-free} data (a).
		\textbf{(e)}
		Predicted weight patches based on the \textit{noisy} data (b).
		\textbf{(f)}
		The labeled (section of the) test image using the predicted weight patches (d).
		Comparing this result to the result depicted by Figure \ref{fig:Experiments_Labeling_uniform}
		shows the effect of \textit{predicted} parameter adaption.
		\textbf{Last row:}
		Further labelings on unseen noisy random test images.
	}\label{fig:Prediction}
\end{figure}

\section{Conclusion and Further Work}
\label{sec:Conclusion}
\subsection{Conclusion}
We presented a novel approach for learning the parameters of the linearized
assignment flow for image labeling. Based on the exact formula of the parameter gradient of a loss function subject to the ODE-constraint, an approximation of the gradient was derived
using exponential integration and a Krylov subspace based low-rank approximation, that is memory efficient and sufficiently accurate.
Experiments demonstrate that our research implementation is on par with highly tuned-machine
learning toolboxes. Unlike the latter, however, our approach additionally returns the essential information for
image labeling in terms of a low-dimensional parameter subspace.

\subsection{Future Work}\label{ssec:Application_Generalized_LAF}
Our future work will study generalizations of the linearized assignment flow. Since this can be done within the overall mathematical framework of the assignment flow approach, the result presented in this paper are applicable. We briefly indicate this for the continuous-time ODE \eqref{eq:LAF-intro} that we write down here again with an index $0$,
\begin{equation}\label{eq:ODE-V0}
\dot V_{0} = A_{0}(\Omega_{0}) V_{0} + B_{0}.
\end{equation}
Recall that $B_{0}$, given by $B_{W_{0}}$ of \eqref{eq:LAF-V-b}, represents the input data \eqref{eq:def-Di} via the mappings \eqref{eq:def-Li} and \eqref{eq:def-Si}. Now suppose the data are represented in another way and denoted by $B_{1}$.  Then consider the additional system
\begin{equation}\label{eq:ODE-V1}
\dot V_{1} = A_{1}(\Omega_{1}) V_{1} + B_{1} + V_{0}(T) L,
\end{equation}
where the solution $V_{0}(T_{0})$ to \eqref{eq:ODE-V0} at time $t=T_{0}$, possibly transformed to a tangent subspace by a linear mapping $L$, modifies the data term $B_{1}$ of \eqref{eq:ODE-V1}. Applying \eqref{eq:V-t-vphi} to \eqref{eq:ODE-V0} at time $t=T_{0}$ and to \eqref{eq:ODE-V1} at time $t=T_{1}$ yields the solution
\begin{equation}
V_{1}(T_{1}) = T_{1}\vphi\big(T_{1} A_{1}(\Omega)\big)\Big(B_{1}+T_{0}\vphi\big(T_{0} A_{0}(\Omega_{0})\big)B_{0} L\Big),
\end{equation}
which is a \textit{composition} of linearized assignment flows and hence linear too, due to the \textit{sequential} coupling of \eqref{eq:ODE-V0} and \eqref{eq:ODE-V1}. \textit{Parallel} coupling of the dynamical systems is feasible as well and leads to larger matrix $\vphi$ that is \textit{structured} and linearly depends on the components $A_{0}(\Omega_{0}), A_{1}(\Omega_{1}), L$. Designing larger networks of this sort by repeating these steps is straightforward.

In either case, the overall basic structure of \eqref{eq:LAF-intro}, \eqref{eq:W-VT} is preserved. This enables us to broaden the scope of assignment flows for applications and to study, in a controlled manner, various mathematical aspects of deep networks in terms of sequences of generalized linearized assignment flow, analogous to \eqref{eq:exp-int-AF-intro}.

\vspace{0.5cm}
\small\noindent
\textbf{Acknowledgement.} This work is funded by the Deutsche Forschungsgemeinschaft (DFG, German Research Foundation) under Germany's Excellence Strategy EXC 2181/1 - 390900948 (the Heidelberg STRUCTURES Excellence Cluster), and within the DFG priority programme 2298 on the ``Theoretical Foundations of Deep Learning'', grant SCHN 457/17-1.
\normalsize

\bibliographystyle{amsalpha}
\bibliography{LAF-Learning}

\appendix
\section{Proofs}
\label{sec:Appendix}

\subsection{Proofs of Section \ref{sec:Closed-Form-Gradient}}\label{app:closed-form-gradient}

\begin{proof}[Proof of Lemma \ref{lem:A1}]
Regarding the differential of the mapping \eqref{eq:def-exp} with respect to its second argument, we have $d\exp_{p}(u)v = R_{\exp_{p}(u)} v$ by \cite[Lemma 4.5]{Zeilmann2020}, with $R$ given by \eqref{eq:def-R-p}. Applying this relation to \eqref{eq:f-A1} where $\exp_{\eins_{\mc{W}}}$ acts row-wise analogous to the mapping $R_{W}$ as explained by \eqref{eq:def-R-Exp-exp-product} and \eqref{eq:RW-SW}, yields
\begin{equation}
df_{1}(\Omega) Y = R_{\exp_{\eins_{\mc{W}}}(-\frac{1}{\rho}\Omega D)}\Big(-\frac{1}{\rho}Y D\Big)
= R_{f_{1}(\Omega)}\Big(-\frac{1}{\rho}Y D\Big),\qquad
\forall Y\in\R^{|I|\times |I|},
\end{equation}
which is \eqref{eq:df-A1-a}. As for the transpose, we vectorize both sides using again \eqref{eq:RW-SW},
\begin{equation}\label{eq:vec-df1-Y}
\vvec_{r}\big(df_{1}(\Omega)Y\big)
= \Diag(R_{f_{1}(\Omega)})\vvec_{r}\Big(-\frac{1}{\rho} Y D\Big)
= -\frac{1}{\rho}\Diag(R_{f_{1}(\Omega)}) (I_{|I|}\otimes D^{\T})\vvec_{r}(Y).
\end{equation}
Applying the transposed matrix to any vector $\vvec_{r}(Z)$ with $Z\in\R^{|I|\times |J|}$ and taking into account the symmetry of the matrix $\Diag(R_{f_{1}(\Omega)})$, yields
\begin{subequations}
\begin{align}
df_{1}(\Omega)^{\T} Z
&= -\frac{1}{\rho}\vvec_{r}^{-1}\big(
(I_{|I|}\otimes D) \Diag(R_{f_{1}(\Omega)}) \vvec_{r}(Z)\big)
\\
\overset{\eqref{eq:RW-SW}}&{=}
-\frac{1}{\rho}\vvec_{r}^{-1}\big((I_{|I|}\otimes D)\vvec_{r}(R_{f_{1}(\Omega)} Z)\big)
= -\frac{1}{\rho}R_{f_{1}(\Omega)}(Z) D^{\T}. \qedhere
\end{align}
\end{subequations}
\end{proof}
\begin{proof}[Proof of Lemma \ref{lem:A2}]
Since $R_{W_{0}}$ does not depend on $\Omega$ and $\vvec_{r}$ is linear, we directly obtain \eqref{eq:df-A2-a}. Regarding the transpose map, we expand the right-hand side of \eqref{eq:df-A2-a},
\begin{equation}
df_{2}(\Omega)Y
\overset{\eqref{eq:RW-SW}}{=}
\Diag(R_{W_{0}})\vvec_{r}(df_{1}(\Omega) Y)
\overset{\eqref{eq:vec-df1-Y}}{=}
-\frac{1}{\rho}\Diag(R_{W_{0}})\Diag(R_{f_{1}(\Omega)}) (I_{|I|}\otimes D^{\T})\vvec_{r}(Y).
\end{equation}
Applying the transposed matrix to any vector $\vvec_{r}(Z)\in\R^{|I|^{2}}$ yields (recall that the matrices $\Diag(R_{W_{0}})$, $\Diag(R_{f_{1}(\Omega)})$ are symmetric)
\begin{subequations}
\begin{align}
d f_{2}(\Omega)^{\T} Z
&= -\frac{1}{\rho}\vvec_{r}^{-1}\big((I_{|I|}\otimes D) \Diag(R_{f_{1}(\Omega)})\Diag(R_{W_{0}})\vvec_{r}(Z)\big)
\\
\overset{\eqref{eq:RW-SW}}&{=}
-\frac{1}{\rho}\vvec_{r}^{-1}\big((I_{|I|}\otimes D)\Diag(R_{f_{1}(\Omega)})\vvec_{r}(R_{W_{0}}Z)\big)
\\
\overset{\eqref{eq:RW-SW}}&{=}
-\frac{1}{\rho}\vvec_{r}^{-1}\Big((I_{|I|}\otimes D)\vvec_{r}\big(R_{f_{1}(\Omega)}(R_{W_{0}}Z)\big)\Big)
= -\frac{1}{\rho}R_{f_{1}(\Omega)}(R_{W_{0}}Z) D^{\T}
\\
\overset{\eqref{eq:df-A1-b}}&{=}
df_{1}(\Omega)^{\T}(R_{W_{0}}Z). \qedhere
\end{align}
\end{subequations}
\end{proof}
\begin{proof}[Proof of Lemma \ref{lem:A3}]
We have
\begin{equation}
df_{3}(\Omega)Y
= \big(d\Diag(R_{f_{1}(\Omega)})Y\big)(\Omega\otimes I_{|J|})
+ \Diag(R_{f_{1}(\Omega)})(Y\otimes I_{|J|}),\quad
\forall Y\in\R^{|I|\times |I|}
\end{equation}
and have to the differential in the first summand on the right-hand side. By \eqref{eq:RW-SW},
\begin{equation}
\Diag(R_{f_{1}(\Omega)})\vvec_{r}(S)
= \vvec_{r}(R_{f_{1}(\Omega)} S),\quad
\forall S\in\R^{|I|\times |J|}
\end{equation}
and hence $d\Diag(R_{f_{1}(\Omega)})$ is given by
\begin{equation}\label{eq:lem-proof-f3-def-eq}
\big(d\Diag(R_{f_{1}(\Omega)}) Y\big)\vvec_{r}(S)
= \vvec_{r}\big((d R_{f_{1}(\Omega)} Y) S\big),
\qquad \forall Y\in\R^{|I|\times |I|},
\quad \forall S\in\R^{|I|\times |J|}.
\end{equation}
It remains to compute $d R_{f_{1}(\Omega)}$ and to evaluate the defining right-hand side, to obtain the left-hand side in explicit form.
Focusing on a single component $R_{f_{1i}}(\Omega)$ of the mapping $R_{f_{1}(\Omega)}$, we have by \eqref{eq:def-R-p}
\begin{subequations}
\begin{align}
R_{f_{1i}}(\Omega)
&= \Diag\big(f_{1i}(\Omega)\big)-f_{1i}(\Omega) f_{1i}(\Omega)^{\T}
\\
d R_{f_{1i}(\Omega)} Y
&= \Diag\big(d f_{1i}(\Omega) Y\big)-\big(df_{1i}(\Omega) Y\big)f_{1i}(\Omega)^{\T} - f_{1i}(\Omega) \big(df_{1i}(\Omega) Y\big)^{\T}
\intertext{
and hence for any $S_{i}\in\R^{|J|}$ and $S = (\dotsc,S_{i},\dotsc)^{\T} \in\R^{|I|\times |J|}$
}
(d R_{f_{1i}(\Omega)} Y) S_{i}
&= \big((d R_{f_{1}(\Omega)} Y) S\big)_{i},\quad i\in I.
\intertext{
Thus, analogous to \eqref{eq:RW-SW}, we obtain
}
(d R_{f_{1}(\Omega)} Y) S
&= \big(\dotsc,(d R_{f_{1i}(\Omega)} Y) S_{i},\dotsc)^{\T}
= \vvec_{r}^{-1}\Big(
\big(\Diag(d R_{f_{1}(\Omega)}) Y\big)\vvec_{r}(S)
\Big).
\intertext{
Applying $\vvec_{r}$ to both sides and comparing with \eqref{eq:lem-proof-f3-def-eq}, we conclude
}
d\Diag(R_{f_{1}(\Omega)}) Y
&= \Diag(d R_{f_{1}(\Omega)} Y)
\end{align}
\end{subequations}
which proves \eqref{eq:df-A3}.
\end{proof}
\begin{proof}[Proof of Lemma \ref{lem:dmcR-Omega}]
The mapping $\exp_{p}$ specified by \eqref{eq:def-exp} satisfies $\exp_{p}=\exp_{p}\circ \Pi_{0}$ and a short computation \cite[Appendix]{Astrom:2017ac}) shows that the restriction $\exp_{p}|_{T_{0}}$, again denoted by $\exp_{p}$,  has the inverse
\begin{equation}
\exp_{p}^{-1}\colon \mc{S}\to T_{0},\qquad
q \mapsto \Pi_{0}(\log q - \log p)
\end{equation}
and consequently the differential
\begin{equation}
d\exp_{p}^{-1}(q) u = \Pi_{0}\Big(\frac{u}{q}\Big),\quad u \in T_{0}.
\end{equation}
For $W, \wt{W}\in\mc{W}$ and $V\in\mc{T}_{0}$, this differential applies componentwise, i.e.~
\begin{equation}
\big(d\exp_{W}^{-1}(\wt{W}) V\big)_{i}
= \Pi_{0}\Big(\frac{V_{i}}{\wt{W}_{i}}\Big),\quad i\in I.
\end{equation}
Application to \eqref{eq:R-Omega} yields for any $Y\in\mc{Y}_{\Omega}$ equation \eqref{eq:dmcR-Omega}.
\end{proof}

\end{document}